\documentclass[11pt,letterpaper]{article} 
\usepackage{microtype}
\usepackage{graphicx}
\usepackage{subfigure}
\usepackage{paralist}
\usepackage{booktabs} 
\usepackage{natbib}
\usepackage{algorithmic}

\usepackage{amsmath,amsfonts,fullpage,url}

\usepackage{hyperref}

\usepackage{url}  

\usepackage[ruled,vlined]{algorithm2e}

\newtheorem{theorem}{Theorem}[section]
\newtheorem{corollary}[theorem]{Corollary}
\newtheorem{lemma}[theorem]{Lemma}

\newtheorem{definition}[theorem]{Definition}

\newcommand{\qed}{\rule{7pt}{7pt}}

\newenvironment{proof}{\noindent{\bf Proof}\hspace*{1em}}{\qed\medskip}
\newenvironment{proof-sketch}{\noindent{\bf Sketch of Proof}\hspace*{1em}}{\qed\medskip}
\newenvironment{proof-idea}{\noindent{\bf Proof Idea}\hspace*{1em}}{\qed\medskip}
\newenvironment{proof-of-lemma}[1]{\noindent{\bf Proof of Lemma #1}\hspace*{1em}}{\qed\medskip}
\newenvironment{proof-of-corollary}[1]{\noindent{\bf Proof of Corollary #1}\hspace*{1em}}{\qed\medskip}
\newenvironment{proof-of-corollary-noqed}[1]{\noindent{\bf Proof of Corollary #1}\hspace*{1em}}{}
\newenvironment{proof-of-claim}[1]{\noindent{\bf Proof of Claim #1}\hspace*{1em}}{\qed\medskip}
\newenvironment{proof-of-claim-noqed}[1]{\noindent{\bf Proof of Claim #1}\hspace*{1em}}{}
\newenvironment{proof-attempt}{\noindent{\bf Proof Attempt}\hspace*{1em}}{\qed\medskip}
\newenvironment{proofof}[1]{\noindent{\bf Proof}
of #1:\hspace*{1em}}{\qed\medskip}

\makeatletter
\def\fnum@figure{{\bf Figure \thefigure}}
\def\fnum@table{{\bf Table \thetable}}
\long\def\@mycaption#1[#2]#3{\addcontentsline{\csname
  ext@#1\endcsname}{#1}{\protect\numberline{\csname 
  the#1\endcsname}{\ignorespaces #2}}\par
  \begingroup
    \@parboxrestore
    \small
    \@makecaption{\csname fnum@#1\endcsname}{\ignorespaces #3}\par
  \endgroup}
\def\mycaption{\refstepcounter\@captype \@dblarg{\@mycaption\@captype}}
\makeatother

\newcommand{\mathify}[1]{\ifmmode{#1}\else\mbox{$#1$}\fi}
\newcommand{\bigO}O

\newcommand{\R}{{\mathbb R}}

\renewcommand{\vec}[1]{{\mathbf #1}}

\providecommand{\norm}[1]{\lVert #1 \rVert}

\newcommand{\Ex}{\mathbb{E}}
\newcommand{\Real}{\mathbb{R}}
\newcommand{\Ss}{S}

\newcommand{\tr}{\mathrm{tr}}
\newcommand{\Ww}{\mathcal{W}}
\newcommand{\Oo}{\mathcal{O}}
\newcommand{\Pp}{\mathcal{P}}
\newcommand{\Hh}{\mathcal{H}}
\newcommand{\Igood}{I_{good}}
\newcommand{\Igoodt}{I^{(\ell)}_{good}}
\newcommand{\rt}{r^{(\ell)}}
\newcommand{\Sep}{S_{\varepsilon}}
\newcommand{\Set}{S_{\varepsilon}^{(\ell)}}
\newcommand{\Sigf}{\sigma_{\nabla f}}

\newcommand{\XI}{\vec{x}^{(i)}}
\newcommand{\YI}{\vec{y}^{(i)}}
\newcommand{\YIT}{\vec{y}^{(i)\top}}
\newcommand{\ZI}{z^{(i)}}
\newcommand{\FI}{f^{(i)}}
\newcommand{\Wst}{\vec{w}^{*\top}}
\newcommand{\Wt}{\vec{w}^{\top}}
\newcommand{\Wstar}{\vec{w}^*}
\newcommand{\Wtrue}{\vec{w}^*_{true}}
\newcommand{\Wh}{\hat{\vec{w}}}
\newcommand{\Wavg}{\hat{\vec{w}}_{avg}}
\newcommand{\Mgood}{t }
\newcommand{\Ngood}{N_{good} }

\title{Conditional Linear Regression}
\date{}

\begin{document}

\author{Diego Calderon\thanks{Part of this work was performed during an REU at Washington University in St.~Louis, supported by WUSEF.}\\ University of Arkansas \\ dacalder@uark.edu 
\and 
Brendan Juba\thanks{Supported by an AFOSR Young Investigator Award and NSF award CCF-718380.} \quad Sirui Li \quad Zongyi Li \\ Washington University in St. Louis \\ \{bjuba, sirui.li, zli\}@wustl.edu 
\and 
Lisa Ruan~\thanks{Part of this work was performed during an REU at Washington University in St.~Louis, supported by the NSF Big Data Analytics REU Site, award IIS-1560191.} \\ M.I.T. \\llruan@mit.edu}

\maketitle

\begin{abstract}
Work in machine learning and statistics commonly focuses on building models that capture the vast majority of data, possibly ignoring a segment of the population as outliers. However, there does not often exist a good model on the whole dataset, so we seek to find a small subset where there exists a useful model. We are interested in finding a linear rule capable of achieving more accurate predictions for just a segment of the population. We give an efficient algorithm with theoretical analysis for the conditional linear regression task, which is the joint task of identifying a significant segment of the population, described by a $k$-DNF, along with its linear regression fit.
\end{abstract}

\section{Introduction}
Linear regression is the task of modeling the relationship between a result variable and some explanatory variables by a linear rule. Linear regression is a standard tool of statistical analysis, with widespread applications spanning essentially all of the sciences.
While the standard linear regression task seeks to model the majority of the data, we consider problems where a regression fit could exist for some subset or portion of the data, that does not necessarily model the majority of the data. We will consider cases in which the subset with a linear model is described by some simple condition; in other words, we desire to perform linear regression on this conditional distribution. Note that neither the condition nor the model is known in advance. We are seeking an algorithm for this task that scales reasonably with the number of predictors used in the model, the dimension of the data, and so on.

To illustrate our problem, consider a set of patient data from a hospital that includes multiple continuous factors, such as rate of smoking,  radiation exposure,  physical activity. Assume we want to predict risk of developing lung cancer. There may be no linear rules that model, e.g., the risk of developing lung cancer for the majority of the data.  However, there may be some linear model that fits a specific subset of the data well, such as adult city-dwellers. If such a model exists, we aim to find it together with the description of the corresponding subset. Our focus is on identifying the portions of the population for which such simply structured models succeed at making accurate predictions, even when such models do not exist for most of the population.

\subsection{Our contributions}

This problem was introduced by \citet{juba17}, who gave an algorithm for conditional linear regression under the $\ell_\infty$-loss where the predictor factors are sparse (i.e., its time and data requirements are exponential in the number of regression factors), and an algorithm for the general case that only identifies a condition describing a small fraction of the optimal condition. The former, sparse algorithm was extended to general $\ell_p$ losses by \citet{hjlw19}.  In this work, we give an algorithm that, under some mild niceness conditions,
\begin{compactitem}
\item {\em only uses polynomial time and data in the dimension and number of factors,}
\item  {\em recovers a condition that covers as much of the distribution as the optimal condition,}  and 
\item {\em approximates an optimal $\Mgood$-term $k$-DNF with $\tilde{\mathcal{O}}(\Mgood \log\log n)$ blow-up of the error. }
\end{compactitem}
We also present some synthetic data experiments illustrating the capabilities of our new algorithm.

\subsection{Related Work}
Our algorithm builds on the {\em list-learning algorithms} due to \citet{lud}. That work aimed to learn about arbitrary small subsets of the data by producing a list of parameter values containing good estimates of good parameters for any small subset. Charikar et al.~could only address tasks such as mean estimation, and as they discuss, could not obtain sufficient accuracy to use their framework for linear regression. 
Our algorithm, like theirs, iteratively computes local estimates for the regression parameters with consideration of their neighbors, and then (re)clusters the terms using their corresponding parameters. In each iteration there is a risk to lose a small fraction of good points. Our primary innovation lies in using a fixed family of subsets (``terms'' of our conditions) as the basic units of data as opposed to individual points. Since these terms are all large enough and we can obtain higher accuracy on these terms given enough data, we can show that the algorithm will not lose any terms with high probability, and it can obtain adequate estimates of the regression parameters. 
Specially, our improvements upon the work of \cite{lud} are:
\begin{compactitem}
\item {We modify Charikar et al's algorithm and analysis to operate on sets instead of points, by introducing weights and modifying the definition of neighbors. }
\item {In our improved algorithm, we strengthen the analysis to show that it does not lose good sets in any iteration, in contrast to Charikar et al's algorithm which indeed may lose a small fraction of good points. Consequently, whereas the original algorithm may need to terminate with a relatively inaccurate estimate of the parameters, we can potentially reach any accuracy we want.}
\item  {We incorporate pre- and post-processing to convert points to our atomic sets, and to use a covering algorithm in the end to extract a $k$-DNF condition on the original data space.}
\end{compactitem}
We stress that Charikar et al.\ (and subsequent works such as \citet{dks18}) cannot obtain a linear predictor with loss that scales with the loss of the best linear predictor on the data (i.e., that goes to zero with the ``noise'' rate) on account of a difference in the formulation: Charikar et al.\ and Diakonikolas et al.\ consider arbitrary subsets of the data, whereas we only consider subsets described by $k$-DNFs. Diakonikolas et al.\ show that even for the simpler problem of mean estimation, one can only guarantee loss that scales polylogarithmically with the density of the set for which we estimate the mean.

Our problem is similar in spirit to work in {\em robust statistics}~\citep{huber81,rl87}, with the key distinction that robust statistics assumes that the outliers comprise a minority of the data. By contrast here, the vast majority of the data may be ``outliers.'' Another setting in this vein is {\em learning with rejection} in which we think of a predictor as having the option to ``abstain'' from making a prediction. In most cases, the strategy for deciding when to abstain is based on some measure of ``confidence'' of the prediction---for example, this is how El-Yaniv and Weiner conceived of such a linear regression task~\citep{eyw12}. The difference is that these methods do not generally produce a ``nice'' description of the region on which they will make a prediction. On the other hand, \citet{cdsm16} considered a version of the task in which the prediction region is constrained to come from a fixed family of nice rules, like our version. The difference is that Cortes et al.~do not seek to achieve given rates of coverage or error, but rather posit that abstaining from prediction has a known, fixed cost relative to the cost of an error, and seek to minimize this overall cost. In our work, we obtain a description of the regions where we will predictor or abstain as a Boolean formula like Cortes et al.~(but not previous works in this area), and also give simultaneous guarantees on the overall coverage and loss unlike Cortes et al., who can only bound a weighted sum of the two. Finally, algorithms such as RANSAC~\citep{fb81} similarly find a dense linear relation among a subset of the points when one exists, but these algorithms scale exponentially with the dimension, and like learning with rejection, do not obtain a rule characterizing which points will satisfy the linear relationship.

There are many other works that fit the data using multiple linear rules, such as linear mixed models, segmented regression, or piecewise linear regression. These methods are similar in that is similar in that the portion of the data fit by an individual linear rule may be small. The distinction is generally that they seek to model the entire data distribution with linear rules, i.e., they cluster the data and minimize the total regression loss over all clusters. By contrast, we only seek a small fraction (say, 10\%) of the data where a good linear regression fit exists. Specifically,
\begin{itemize}
\item {\em Linear mixed models}~\citep{mccs01,jiang07} simply view the data set as a mixture of data fitting linear rules. Such work usually assumes that all or most of the data belongs to one of the mixture components. Moreover, in such models (as with learning with a reject option) the components do not come attached with a rule describing their domain. To use such models for prediction, one has to try the various linear rules to see which gives a small residual, which may not be so well behaved.

\item {\em Generalized cluster-wise linear regression}~\citep{park16} describes data in terms of ``entities,'' where the goal is to partition these entities into clusters so that the overall total regression loss is minimized. Thus, unlike our work, these methods seek to fit the entire data set with a linear rule in some cluster. It is similar in that our terms, defined by Boolean attributes, can be viewed as entities in their setting. But, since each combination of Boolean attributes defines a term, the total number of terms could be very large (say, a million), and all these terms' data points are defined in the same space, so they are also quite different in practice.

\item In {\em Gaussian process classification using random decision forests}~\citep{frohlich12}, Fr{\H{o}}hlich et al.~propose to learn a random decision forest whose leaf nodes use a Gaussian process classifier. This algorithm also views the data as having been sampled from a linear mixture model, where again we only seek to find a small cluster. Again, this method does not produce a nice rule describing which rule to use for prediction.

\item Work on {\em regression trees} such as by \citet{quinlan92} may also be seen as finding a family of nice rules (i.e., the branches of the tree) such that on the partitions described by these rules, the data is fit by linear predictors. Again, we differ in not seeking to find linear predictors for the entire data distribution. Furthermore, while conjunctive splits are surely nice, it seems to be intractable to identify good conjunctions, even if we are only seeking one; see \citet{juba17} for details. In any case, no guarantees are known for such methods, and it is natural to conjecture that they cannot be guaranteed to work.
\end{itemize}



\section{Preliminaries and Definitions}


We suppose we have a data set consisting of $N$ examples, where each example has three kinds of attributes: a vector of $n$ Boolean attributes $\vec{x}$, a vector of $d$ real-valued attributes $\vec{y}$, and a real-valued target attribute $z$ that we wish to predict. For example, in our cancer prediction setting, we have an example $(\vec{x}, \vec{y}, z)^{(i)}$ for each $i$th patient in which:

\begin{compactenum}
    \item $\vec{x}^{(i)}$ is a vector of Boolean demographic properties that describe patient $i$ (e.g., adult, city dweller)
    \item $\vec{y}^{(i)}$ is vector of continuous risk factors for patient $i$, (e.g., rate of smoking, radiation exposure, physical activity)
    \item $z^{(i)}$ represents the variable of interest such as probability of developing a certain kind of cancer in the next ten years.
\end{compactenum}
When there is no confusion, we will use $\XI$ to denote the whole point $(\vec{x}, \vec{y}, z)^{(i)}$.

For example, $(y_1,y_2,y_3,z)^{(2)}=(60,100,18,.1)$ could represent patient 2 with height 60 inches, weight 100 lbs, age 18, and a 10\% probability of developing cancer in ten years.

Note that we can obtain Boolean attributes $\vec{x}$ from continuous attributes $\vec{y}$ by using binning and splits, similar to decision trees. For example, we can define a new Boolean attribute such as $x= $ ``$y \leq a$?'' for some quantile of the data $a$. Conversely, Boolean variables can also be viewed as regression factors.

We wish to find a linear rule $\langle \vec{w},\vec{y}\rangle$ to predict $z$. We would typically achieve this by minimizing the loss $\| \langle \vec{w}, \vec{y} \rangle -z\|_2$ averaged over the data. But, it's common that there doesn't exist a good linear rule for the whole data set. We propose to find a subset (a condition or ``cluster'' $\vec{c}$), such that there exists good fit on $\vec{c}$.
Of course, if we just pick any subset that fits some linear rule well, this is unlikely to be predictive. Instead, we will seek to pick out a subset according to some simple criteria. In other words, this condition should be described by some simple rule which will use the Boolean $\vec{x}$ attributes. Following the previous work~\cite{juba17}, we will use conditions represented by $k$-DNF (Disjunctive normal form) formulas, which is an ``or'' of {\em terms}, where each term is an ``and'' of at most $k$ attributes (which we permit to be negated). Thus:
\begin{align*}
\vec{c} &= t_1 \vee \ldots \vee t_s \  \text{for some }s\text{, where}\\
t_i &= \ell_{i,1} \wedge \ldots \wedge \ell_{i,k}, \ \text{and where } \ell_{i,j}\text{ is either }x_{i,j}\text{ or }\neg x_{i,j}, \text{ for some attribute } x_{i,j}
\end{align*}
\noindent
and when we say $t$-term DNF, we mean $s \leq t$. We focus on $k$-DNF conditions since the use of other natural representations seems to result in an inherently intractable problem~\citep{juba17}.

To summarize, we want to find a $k$-DNF condition such there exists a good linear prediction rule on the data satisfying this condition. Meanwhile, we want this condition to describe a subset of the data that isn't too small. We will demand that data belongs to it with probability at least $\mu$, as shown in Figure \ref{clr}. Since the task is thus to choose an appropriate set of terms (defining a $k$-DNF), we can view the terms as $m$ atomic sets of data.\footnote{%
Indeed, we will not use the structure of the terms and we could work instead with an arbitrary family of $m$ such atomic subsets of the data, as long as we can collect enough data relative to the number of subsets.}

\begin{figure}[h]
\includegraphics[width=4cm]{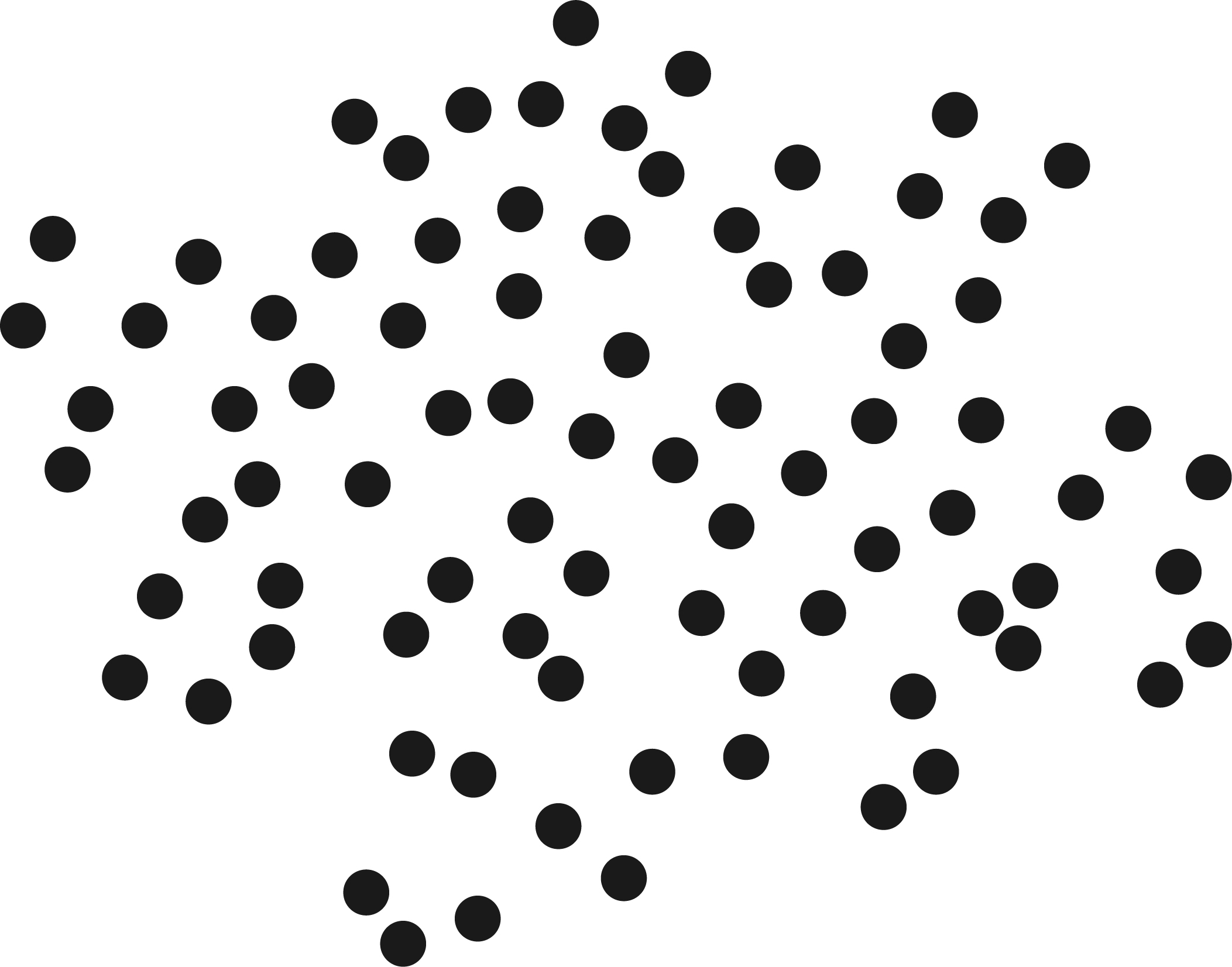}
\includegraphics[width=4cm]{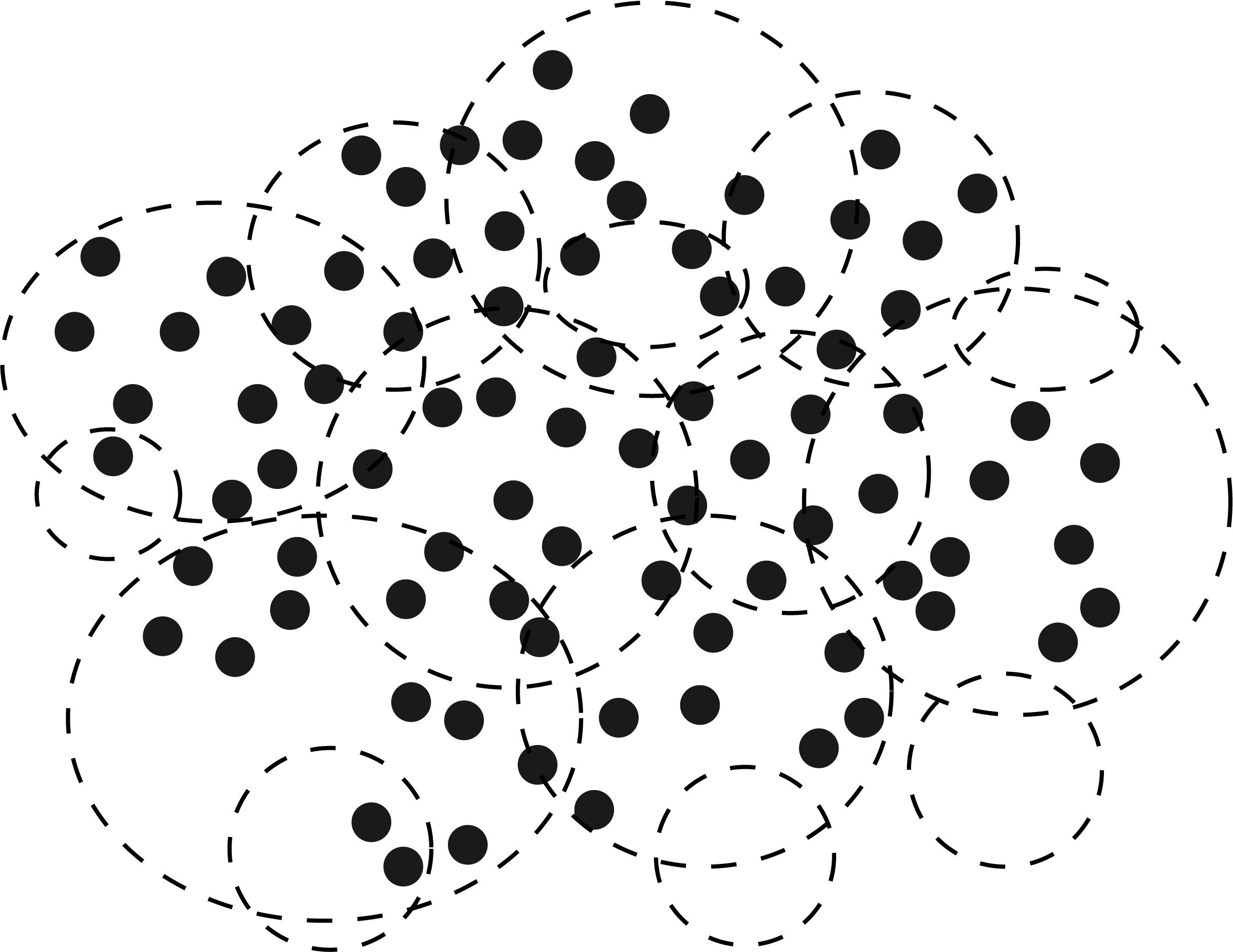}
\includegraphics[width=4cm]{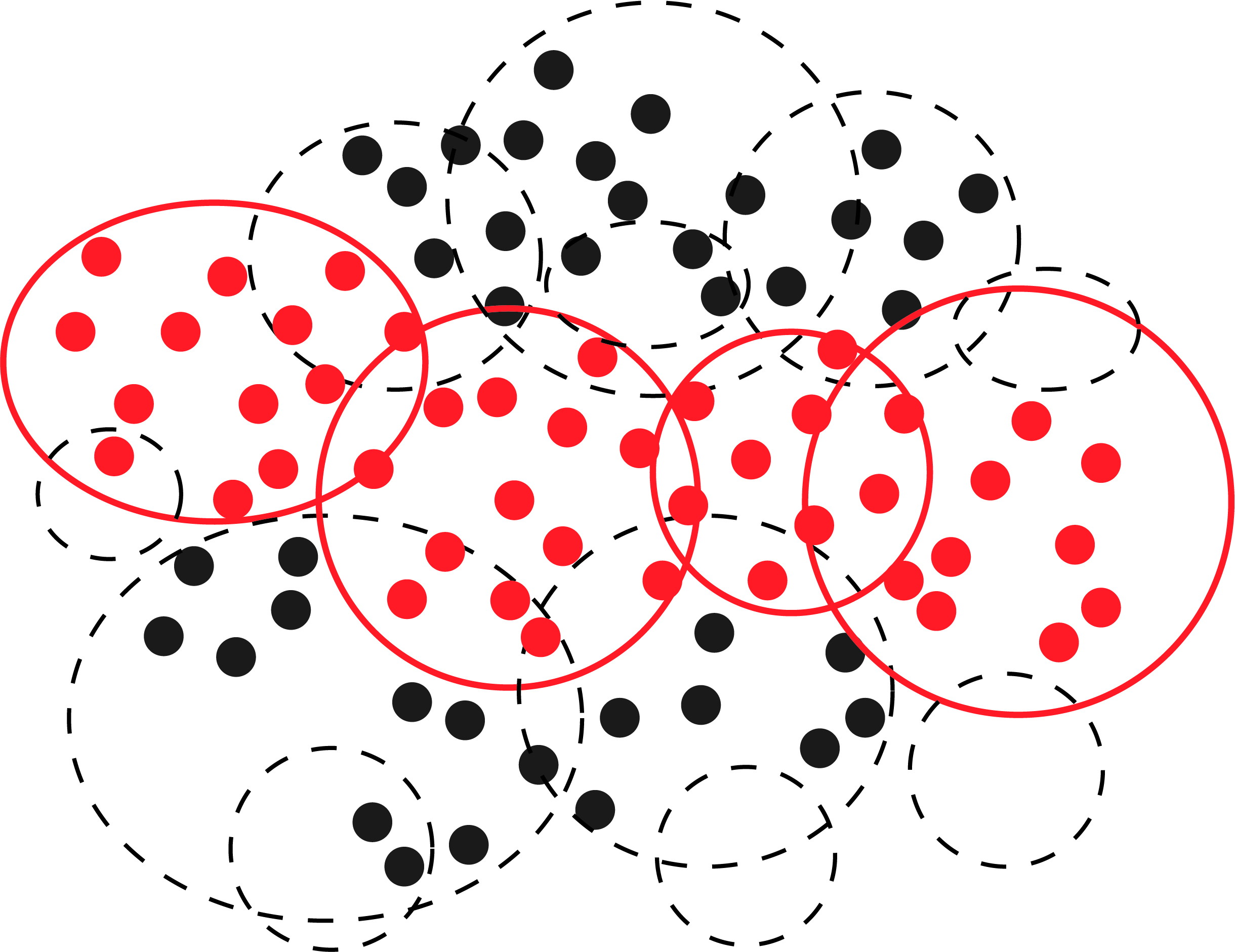}
\centering
\caption{Conditional Linear Regression}\label{clr}
\small{
Data 
$\{(x_1,\ldots,x_{n},y_1,z)\}$ on the $y_1 \times z$ plane. Each term $t$ is depicted in the $y_1\times z$ plane as a circle, enclosing the points satisfying that term. A $k$-DNF is then a union of these circles.
\textbf{1:} The data space.
\textbf{2:} The data and the terms.
\textbf{3:} A selected subset on the data space.
}\raggedright
\end{figure}

\begin{definition}[Conditional Linear Regression]
Conditional linear regression is the following task: 
given data $ \{(x^{(i)}_1,\ldots,x^{(i)}_n,y^{(i)}_1,\ldots,y^{(i)}_d,z^{(i)})\}_{i=1}^N$ drawn i.i.d. from a distribution $D$, we are to find a k-DNF condition $\vec{c}$ and a regression fit $\vec{w}= (w_1,\ldots,w_d)$, such that:\\ 
 \hspace*{0.5cm} (1) The regression loss $\|\langle \vec{w}, \vec{y} \rangle -z\|$ is bounded for the points satisfying $\vec{c}$.\\
 \hspace*{0.5cm} (2) The condition has probability at least $\mu$:
$\Pr[\vec{c}(\vec{x}) = 1] \geq \mu$.
\end{definition}

We will describe an algorithm that can return a pair $(\hat{\vec{c}},\hat{\vec{w}})$ that is close to the potential optimal solution $(\vec{c}^*,\vec{w}^*)$, given that the data distribution on $\vec{c}$ is sufficiently nice in the following sense. First, we assume that $\vec{w}^*$ gives a predictor with subgaussian residuals on $\vec{c}$. Second, we will assume that our loss function is Lipschitz. Third, we will consider the following quantity measuring how the desired conditional distribution $D|\vec{c}$ varies across the terms of $\vec{c}$: if we abuse notation to let a DNF $\vec{t}$ also denote a set of terms (thus, letting $|\vec{t}|$ denote the number of terms in $\vec{t}$),
\[
{\Sep}_0  :=  \max_{\vec{t}\subseteq\vec{c}:\Pr[\vec{t}(\vec{x})|\vec{c}]\geq\varepsilon}\frac{1}{\sqrt{|\vec{t}|}}\Big\|\Big[ \Ex[\vec{y}\vec{y}^\top | t_j] -  \mathbb{E}[\vec{y} \vec{y}^{\top}| \vec{c}]\Big]_{t_j\in \vec{t}} \Big\|_{op}.
\]
${\Sep}_0$ measures, in a spectral sense, how different on average the distributions of sufficiently large sets of terms are from the distribution over the entire desired subset $\vec{c}$. For example, if the distribution on this subset is independent of which term of $\vec{c}$ is satisfied, then ${\Sep}_0$ is $0$. Intuitively, if the distribution across these terms is very different from one another, it will be harder for our algorithm to identify that they should be fit by a common linear rule. We will need that this quantity is sufficiently small relative to the degree of (strong) convexity of the loss. More specifically, recall:
\begin{definition} A function $f:\Hh\to\Real$ for $\Hh\subseteq \Real^d$ is $\kappa-$strongly convex if $\forall \vec{w},\vec{w}' \in \Hh$,
\[f(\vec{w}') \geq \langle \vec{w}'-\vec{w}, \nabla f(\vec{w}) \rangle + \frac{\kappa}{2}\|\vec{w}'-\vec{w}\|_2^2. \]
\end{definition}
Generally, for the squared error loss on a bounded space, the convexity coefficient can be viewed as a constant, when the bound $\|y\|_2 \leq B$ is fixed. Note that we will be able to enforce that our loss function is at least $\kappa$-strongly convex by adding a regularization term, $\frac{\kappa}{2}\|\vec{w}\|_2^2$, at the potential cost of reducing the quality of the solution.
Our main theorem is of the form:

\begin{theorem}[Conditional $l_2$-Linear Regression]\label{2.2}
Suppose $D$ is a joint distribution over $x \in \{0,1\}^{n}, y \in \mathcal{B} \subset \mathbb{R}^{d}$ and $z \in\mathbb{R}$.
 If there exists a (ideal) $\Mgood$-term k-DNF cluster $\vec{c}^*$ and regression fit $\vec{w}^* \in \Hh \subset \mathbb{R}^{d_y}$, where $\mathcal{B}$ has $l_2$ radius $B$ and $\Hh$ has $l_2$ radius $r$, such that:\\
\hspace*{0.5cm}(1): $\mathbb{E}_D[( \langle \vec{w}^*,\vec{y} \rangle -z )^2  | \vec{c}^*(\vec{x}) =1]   \leq \epsilon.$\\ 
\hspace*{0.5cm}(2): $Pr[\vec{c}^*(\vec{x})=1]\geq \mu$.\\
\hspace*{0.5cm}(3): the error $( \langle \vec{w}^*,\vec{y} \rangle -z )$ follows a $\sigma$-subgaussian distribution on $D|\vec{c}$\\
\hspace*{0.5cm}(4): the loss function $f(\vec{w},(\vec{y},z)) = ( \langle \vec{w},\vec{y} \rangle -z )^2+\frac{\kappa}{2}\|\vec{w}\|_2^2$ is $L$-lipschitz and $\kappa$-strongly convex, where $\kappa\geq\Omega\Big(\frac{\Mgood S_{\frac{\gamma}{t}0}\log\frac{1}{\mu}}{\sqrt{\mu}}\Big)$\\
then for $\delta, \gamma \in (0,1)$, using
$N = \Oo(\frac{B^6 d^3\sigma^2L^2t^2}{\mu\gamma^4}\log(m/\delta) ) )$ examples,
we can find a k-DNF $\hat{\vec{c}}$ and parameters $\hat{\vec{w}}$ in polynomial time, such that with probability $(1-\delta)$:\\
\hspace*{0.5cm}(1): $\mathbb{E}_D[( \langle \hat{\vec{w}},\vec{y} \rangle -z )^2  | \hat{\vec{c}}(\vec{x}) =1]   \leq \Oo (\Mgood\log(\mu N)(\epsilon+\gamma))$\\
\hspace*{0.5cm}(2): $ Pr[\hat{\vec{c}}(\vec{x})=1]\geq (1-\gamma)\mu $.
\end{theorem}
We remark that the Lipschitz condition follows from the bound on $y$ and $w$, $L \leq rB^2$,
and the number of good terms $t$ is always at most the total number of possible terms $ m \leq n^k$, so $\Oo (\Mgood\log(\mu N)\epsilon) \leq \tilde{\Oo} (n^k\epsilon)$ (suppressing the other parameters).

The ideal subset $\{x: \vec{c}^*(\vec{x}) =1\}$ is denoted as $\Igood$.
When later working on the space of terms, we also use $\Igood$ to denote the collection of terms of the DNF $\vec{c}^*$: $\{t_i : t_i \text{ is a term of } \vec{c}^* \}$, so the number of terms in $\Igood$ is $\Mgood$.
From the perspective of \citet{lud}, we treat $\Igood$ as our ``good data,'' with the other points being arbitrary bad data. Our algorithm is going to suggest a list of candidate parameters $\hat{\vec{w}}$, with one of them approximating $\vec{w}^*$.

\section{Soft Regression and Outlier Removal}
Our algorithm works primarily on terms: we consider the terms to be atomic sets of data, whose weights are the number of points (probability mass) satisfying the terms.
The main idea of our algorithm is to compute regression parameters $\vec{w}_j$ for each term $t_j$, and cluster the {\em terms} by the distance of their {\em parameters}, similar to \citet{lud}. Towards realizing this strategy, we need to compute approximations to the regression parameters that are not too impacted by the presence of bad data (``outliers'').

\subsection{Preprocessing}
In this section, we show how to convert the data into a suitable form: later, we will assume the terms are disjoint and that we have an adequate number of examples to estimate the loss on each term. We will ensure these conditions by introducing duplicate points when they are shared, and by deleting terms that are satisfied by too few examples.

\subsubsection{Reduction to Disjoint Terms by Duplicating Points}
Given $N$ data points and $m$ terms $ t_1, \ldots, t_m$, if we view terms as sets, our analysis will require these terms to be disjoint. A simple method is to duplicate the points for each term they are contained in. 
For example, if the $i^{th}$ point $\XI = (\vec{x},\vec{y},z)^{(i)}$ is contained in terms $t_a$ and $t_b$, then we create two points $(\vec{x},\vec{y},z)^{(a,i)}$ and $(\vec{x},\vec{y},z)^{(b,i)}$, each with the same attributes $ (\vec{x},\vec{y},z)$ as the original point $\XI$. 
After duplication, the terms are disjoint, and there will be at most $Nm$ points. We denote the resulting number of points by $N'$.
The size of $\Igood$, changing from $|\bigcup_{\Igood} t_i| $ to $\sum_{\Igood} |t_i| $, may also blow up with a factor ranging from $1$ to $\Mgood$. 
Note that the proportion of good points $N_{good}/N$ decreases by at most a factor of $1/m$ since $N'_{good}/N' \geq N_{good}/mN$. This double counting process may skew the empirical distribution of $\Igood$ by up to a factor of $\Mgood$. Consequently, it may result in up to a $\Mgood \leq n^{k}$-factor blow-up in the error, and this is ultimately the source of the increase in loss suffered by our algorithm in Theorem~\ref{2.2}.
For convenience, we will use the same notation $N$, $\Igood$ and $\mu$ for both before and after duplication when there is no confusion. 

\subsubsection{Reduction to Adequately Sampled Terms by Deleting Small Terms}
The approach of \citet{lud} can only guarantee that we obtain satisfactory estimates of the parameters for sufficiently large subsets of the data.  Intuitively, this is not a significant limitation as if a term has very small size, it will not contribute much to our empirical estimates. Indeed, with high probability, the small terms (terms with size $< \varepsilon \mu N $ for  $\varepsilon\leq \gamma/\Mgood$) only comprise a $\gamma$ fraction of $\Igood$. Based on this motivation, if a term has size less than $\varepsilon \mu N$, then we just delete it at the beginning (note here $\varepsilon$ meaning for a fraction of data, should not be confused with $\epsilon$ for error). Especially for a $\Mgood$-term DNF, not many terms could be small, so it is safe to ignore these small terms. As before, we will continue to abuse notation, using $t$ and $m$ for the number of terms when there is no confusion.

\subsection{Loss Functions}
In this section we define our loss functions and analyze their properties.

Given $N$ data points and $m$ disjoint sets (terms) $t_1 , \ldots, t_m$ with size (weight) $|t_1| , \ldots, |t_m| $, we can define a loss function for each point in the space of parameters. For each $i$th point, define $f^{(i)}: \Hh \rightarrow \Real$ by \[\FI(\vec{w}) = (\ZI - \langle \vec{w}, \YI \rangle )^2\]
Similarly, we define a loss function for each of the terms $t_j$, 
$f_1,\ldots, f_m: \Hh \rightarrow \Real$, as the average loss over these data points $\{\XI=(\XI, \YI, \ZI)\}$ in the term $t_j$ (beware we abuse the notation to let $\XI$ denote the $i$th point).
\begin{align*}
    f_j(\vec{w}) 
    &= \frac{1}{|t_j|}\sum_{ \XI \in t_j} \FI(\vec{w}) \\
    &= \frac{1}{|t_j|}\sum_{ \XI \in t_j} (\ZI - \langle \vec{w}, \YI \rangle )^2 \\
    &\underset{(*)}{=} \frac{1}{|t_j|}\|\vec{z} - Y \Wt \|^2_2\\
    &= \frac{1}{|t_j|} (\vec{z} - Y \Wt)^{\top} (\vec{z} - Y \Wt)\\
    &= \frac{1}{|t_j|} \left(\vec{z}^{\top}\vec{z} - \vec{z}^{\top}Y \Wt - \vec{z}Y^{\top}\vec{w} + \vec{w}Y^{\top}Y\Wt \right)\\
    &= \frac{1}{|t_j|} [1,\vec{w}]  
        \begin{bmatrix}
            \vec{z}^{\top}\vec{z} &  -\vec{z}^{\top} Y  \\
             -Y^{\top}\vec{z} &  Y^{\top}Y
        \end{bmatrix} 
    [1,\vec{w}]^{\top}
\end{align*}
Where at $(*)$, we write the formula in vectors and matrices. We treat $\vec{z}$ as a $|t_j| \times 1$ column vector, whose each coordinate is the $z$ for each point in the term $t_j$. Similarly, $Y$ is a $|t_j| \times d$ matrix, each row for a point and $\vec{w}$ is $1 \times d$ row vector. 
One of the advantage of our formulation is that the loss function for each term can be eventually written as $f_j (\vec{w}) = [1, \vec{w}] A [1,\vec{w}]^{\top}$, where $A$  is a $(d+1) \times (d+1)$ matrix. We can pre-compute this quadratic loss matrix $A$ so that the running time of the main algorithm is independent of the number of data points, and is thus a function only of the number of terms and dimension for our regression problem.

Note these loss functions are stochastic, depending on the sample from the distribution $(\vec{x},\vec{y},z) \sim D$. That is, the true loss for a fixed term $t_j$ is:
\[ \Ex [f_j(\vec{w})] = {\Ex}_{\XI} [(\ZI - \langle \vec{w}, \YI \rangle )^2 |t_j]. \]
Similarly, for $\Igood$, we define the loss function
\[ f_{\Igood}(\vec{w}) 
    = \frac{1}{|\Igood|}\sum_{ \XI \in \Igood} \FI(\vec{w}). \]
Let $\bar{f}$ denote the expected loss function for points averaged over $\Igood$,
\begin{align*}
\bar{f}(\vec{w}) = \Ex[f_{\Igood}].   
\end{align*}
Then the optimal $\Wstar$ is defined as 
\[ \Wstar := \underset{\vec{w}}{\arg\min} \bar{f}(\vec{w}).\]
Our ultimate goal is to find $\Wh$ that minimizes $\bar{f}(\hat{\vec{w}})$, but the difficulty is that $\bar{f}$ is unknown (since $\Igood$ is unknown). To overcome this barrier, instead of directly minimizing $\bar{f}(\hat{\vec{w}})$,
we try to find a parameter $\hat{\vec{w}}$ such that $\bar{f}(\hat{\vec{w}}) - \bar{f}(\vec{w}^*)$ is small. Once we get a close approximation $\hat{\vec{w}}$, we can use a greedy covering algorithm to find a good corresponding condition $\hat{\vec{c}}$.

In summary, we reformulate our problem in terms of these new loss functions as follows:
\begin{definition}[Restatement of conditional linear regression problem]
Given $D$ a linear regression distribution over points $\{\vec{x}^{(i)}:=(\XI, \vec{y}^{(i)},z^{(i)})\}_{(i)=1}^N$, and $\{t_j\}_{j=1}^m$ predefined disjoint subsets (terms), 
let $\Igood$ be the (unknown) target collection corresponding to $\vec{c}^* = \bigcup_{t_j\in \vec{c}^*} t_j$ with probability mass $Pr[\vec{x} \in \Igood]\geq \mu$, and $\bar{f}$ be the regression loss over $\Igood$. If there exists a linear regression fit $\vec{w}^*$ such that:
\[\mathbb{E}_D[\bar{f}(\vec{w}^*)]  \leq \epsilon.\] 
Then we want to find a $\hat{\vec{w}}$ that approximates $\vec{w}^*$: 
\[\mathbb{E}_D[\bar{f}(\hat{\vec{w}})]  \leq \gamma + \epsilon.\]
\end{definition}

\subsection{Main Optimization Algorithm}

\begin{figure}[h]
\includegraphics[width=15cm]{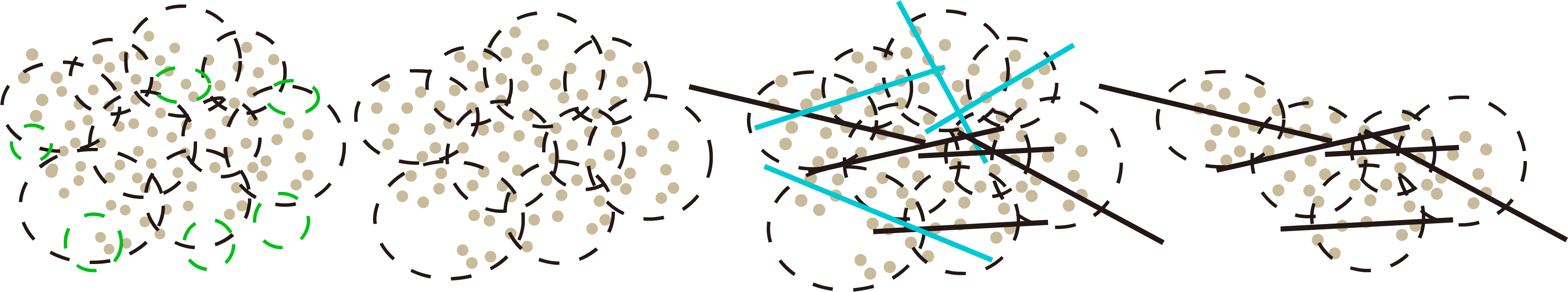}
\centering
\caption{Algorithm \ref{algo1} }\label{pic2}
\small{
\textbf{1:} The original data space with the terms.
\textbf{2:} Delete the small terms and duplicate points.
\textbf{3:} Compute the best regression parameter $\Wh_i$ for each term.
\textbf{4:} Meanwhile iteratively downweight these terms whose $\Wh_i$ have large error on their neighbor terms.
}\raggedright
\end{figure}

The main algorithm is an alternating-minimization-style algorithm: given a soft choice of which terms are outliers, we let each term choose a local set of regression parameters that are collectively regularized by the trace of their enclosing ellipsoid. Then, given these local regression parameters, we update our scoring of outliers by examining which terms find it difficult to assemble a coalition of sufficiently many ``neighboring'' terms whose parameters are, on average, close to the given term. We repeat the two until we obtain a sufficiently small enclosing ellipsoid for the collection of regression parameters.

\subsubsection{Semidefinite Programming for Soft Regression}

Following \citet{lud}, we now present Algorithm \ref{algo1} for approximating the regression parameters. We assign ``local''  regression parameters $\vec{w}_i$ for each term $t_i$, and use a semi-definite program (SDP) to minimize the total loss $\sum |t_i|f_i(\vec{w}_i)$ with regularization to force these parameters to be close to each other. Following each iteration, we use Algorithm \ref{UpdateWeights} to remove outliers, by decreasing the weight factors $c_i$ for those terms without enough neighbors. The process is illustrated in Figure~\ref{pic2}. Intuitively, if there exists a good linear regression fit $\vec{w}^*$ on $\Igood$, then for each term $t_i \in \Igood$, $f_i(\vec{w}^*)$ should be small. Therefore, we can find a small ellipse $Y$ (or $  \mathcal{E}_{Y}$) bounding all parameters for the terms in $\Igood$ if the center of $Y$ is close to $\vec{w}^*$. The SDP will find such an ellipse bounding the parameters while minimizing the weighted total loss.

\begin{algorithm}[h]
\caption{Soft regression algorithm}\label{algo1}
\begin{algorithmic}
   \STATE {\bfseries Input:} terms $t_{1:m}$
    \STATE {\bfseries Output:} parameters $\hat{\vec{w}}_{1:m}$ and a matrix $\hat{Y}$
   \STATE Initialize $c_{1:m} \gets (1,\ldots, 1) $, $\lambda \gets \frac{\sqrt{8\mu} N \Mgood \Ss}{r}$
   \REPEAT
   \STATE Let $\hat{\vec{w}}_{1:m}, \hat{Y}$ be the solution to SDP:
   \begin{equation}\label{SDP}
       \begin{aligned}
        & \underset{\vec{w}_1,\ldots,\vec{w}_m,Y}{\text{minimize}}
        & & \sum_{i=1}^{m} c_i |t_i| f_i (\vec{w}_i) + \lambda \tr(Y)\\
        & \text{subject to}
        & & \vec{w}_i {\vec{w}_i}^{\top} \preceq Y \ \mathrm{for} \ \mathrm{all} \ i=1,\ldots,m.  
        \end{aligned}
   \end{equation}
   \IF{ $\tr(\hat{Y}) > \frac{6r^2}{\mu}$ }
        \STATE $ c \gets \text{UpdateWeights} (c,\hat{\vec{w}}_{1:m}, \hat{Y})$
  \ENDIF
  \UNTIL $\tr(\hat{Y}) \leq \frac{6r^2}{\mu}$
  \STATE {\bfseries Return} $\hat{\vec{w}}_{1:m}, \hat{Y}$
\end{algorithmic}
\end{algorithm}

Formally, in the SDP \ref{SDP} (in Algorithm~\ref{algo1}), 
$Y$ is a $d \times d$-dimensional matrix (recall $d$ is the dimension for $\vec{y}$ and $\vec{w}$). We bound the parameters $\vec{w}_i$ with the ellipse $Y$ by imposing the semidefinite constraint $\vec{w}_i {\vec{w}_i}^{\top} \preceq Y$, which is equivalent to letting 
$\begin{bmatrix}
Y & \vec{w}_i \\
{\vec{w}_i}^{\top} & 1
\end{bmatrix} 
\succeq 0,
$ 
saying that $\vec{w}_i$ lies within the ellipse centered at 0 defined by $Y$.
The regularization $tr(Y)$ of the SDP penalizes the size of the ellipse, making the various parameter copies $\vec{w}_i$ lie close to each other.

\subsubsection{Removing Outliers}
The terms not in $\Igood$ may have large loss for the optimal parameters $\vec{w}^*$, and therefore make the total loss in SDP \ref{SDP} large. To remove these bad terms, we assign a weight factor $c_i \in (0,1)$ for each term $t_i$ and down weight these terms with large loss, as shown in Algorithm \ref{UpdateWeights}.

\begin{algorithm}
\caption{Algorithm for updating outlier weights}\label{UpdateWeights}
\begin{algorithmic}
   \STATE {\bfseries Input:} $c, \hat{\vec{w}}_{1:m}, \hat{Y} $.
    \STATE {\bfseries Output:} $c'$.
    \FOR{$i=1$ {\bfseries to} $m$}
    \STATE Let $\tilde{\vec{w}}_i$ be the solution to
    \begin{equation}
    \begin{aligned}
    & \underset{\tilde{\vec{w}}_i,a_{i1},\ldots,a_{im}}{\text{minimize}}
    & & f_i(\tilde{\vec{w}}_i)\\
    & \text{subject to}
    & & \tilde{\vec{w}}_i = \sum_{j=1}^{m} a_{ij} \hat{\vec{w}}_j,\quad   \sum_{j=1}^{m} a_{ij} = 1 \\
    & & &   0 \leq  a_{ij} \leq \frac{2}{\mu N}|t_j|, \quad \forall j   
    \end{aligned}
    \end{equation}
    $z_i \gets f_i(\tilde{\vec{w}}_i) - f_i(\hat{\vec{w}}_i)$
    \ENDFOR
    \STATE $ z_{max} \gets max\{\ z_i\ \big|\ c_i \neq 0\}  $
    \STATE $c'_i \gets c_i . \frac{z_{max} - z_i}{z_{max}}$ for $i = 1,\ldots ,n$
    \STATE {\bfseries Return}   $c'$
\end{algorithmic}
\end{algorithm}

In Algorithm \ref{UpdateWeights}, we solve an SDP for each term to find its best $\mu N$ neighbor points and compute the ``average'' parameter $\tilde{\vec{w}}_i$ over the neighborhood. $\tilde{\vec{w}}_i$ is a linear combination of its neighbors' parameters: $\tilde{\vec{w}}_i = \sum_{j=1}^{m} a_{ij} \hat{\vec{w}}_j$, minimizing the term's loss $f_i(\tilde{\vec{w}}_i)$.
Intuitively, if a term is a good term, i.e. $t_i \in \Igood$, then its parameter $\hat{\vec{w}}_i$ should be close to the average of parameters of all terms in $\Igood$, $\vec{w}_i \approx \sum_{\Igood}\frac{|t_j|}{|\Igood |}  \vec{w}_j $. 
In the SDP for $t_i$, we define coefficients $a_{ij}$ to play the role of $\frac{|t_j|}{|\Igood |}$. These coefficients $\{a_{ij}\}$ are required to sum to 1, i.e. $\sum_{j=1}^{m} a_{ij} = 1$, and each should not be larger than $\frac{|t_j|}{|\Igood |} \sim 2\frac{|t_j|}{\mu N}$.  At a high level, the SDP computes the best neighbors for $t_i$ by assigning $\{a_{ij}\}$, so that the average parameter $\tilde{\vec{w}}_i$ over the neighbors minimizes $f_i$.
If a term is bad, it is hard to find such good neighbors, so if the loss $f_i(\tilde{\vec{w}}_i)$ is much larger than the original loss, then we consider the term to be an outlier, and down-weight its weight factor $c_i$.

\subsection{A Bound on the Loss That Is Linear in the Radius of Parameters}
Similarly to \citet{lud}, we obtain a theorem saying the algorithm will return meaningful outputs on $\Igood$.
The main change is that we use terms instead of points. In other words, we generalize their arguments from unit-weight points to sets with different weights. And based on a spectral norm analysis, we show the bound will shrink linearly with the radius of the set of candidate parameters, as long as we have enough data.

First, to estimate the losses by their inputs,
we introduce the gradient $\nabla f$. By the convexity of $f$, we have
$(f(\vec{w})-f(\vec{w}^*) )
\leq  \langle \nabla f(\vec{w}), \vec{w}-\vec{w}^*  \rangle$.
Note that $\|\vec{w}-\vec{w}^*\|$ is bounded by $2r$, where $r:= \max \|\vec{w}\|_2$. We will need to bound the gradient as well.

To bound the loss functions, we use the spectral norm of gradients:
\[S := \max_{\vec{w}\in\Hh} \frac{1}{\sqrt{\Mgood}}
\| \big[\nabla f_j(\vec{w}) - \nabla \bar{f}(\vec{w}) \big]_{j\in \Igood} \|_{op}
\]
where $\| \cdot \|_{op}$ is the spectral norm (operator 2-norm) of the matrix whose rows are gradients of loss functions in $\Igood$: $(\nabla f_i(\vec{w}) - \nabla \bar{f}(\vec{w}))_{i\in \Igood}$. $S$ measures the difference between the gradient of loss functions of terms in $\Igood$: $\nabla f_i(\vec{w})$ and gradient of average loss on $\Igood$: $\nabla \bar{f}(\vec{w})$. At a high level, this bound tells us how bad these loss functions could be. We note that since the gradient is a linear operator, this quantity is invariant to regularization of the loss functions.

As shown by \citet{lud}, if $\nabla f_i - \nabla \bar{f} $ is a $\Sigf$ sub-gaussian distribution, then $S = \Oo(\Sigf)$, generally a constant. Although a constant bound $\Oo(\Sigf)$ is good for their purposes -- mean estimation -- it is too weak for linear regression. In the sequel we will show $S$ is going to shrink as the radius of parameters $r$ decreases. 

For linear regression, $f_j(\vec{w}):=
    \frac{1}{|t_j|}\sum_{\XI \in t_j} \FI (\vec{w})$,
and  $\nabla f_j(\vec{w}) = \frac{1}{|t_j|}\sum_{\XI \in t_j} \nabla \FI(\vec{w})$, where for each point
$\nabla \FI(\vec{w}) =  2(\Wt \YI - \ZI) \YI $. 
If we assume $\ZI = \Wst \YI + \epsilon^{(i)} $, and the residual $\epsilon^{(i)}$ (a subgaussian, e.g., from $ \mathcal{N}(0,\sigma_{\epsilon}^2)$) is independent of  $\YI$, then 
\begin{align*}
&\nabla \FI(\vec{w}) = 2(\vec{w}^{\top} - \vec{w}^{*\top}) \YI \YIT + \epsilon^{(i)} \YI \\
&\nabla f_j(\vec{w}) = \frac{1}{|t_j|} \left( 2(\vec{w}^{\top} - \vec{w}^{*\top}) \sum_{\XI \in t_j} \YI \YIT +
 \sum_{\XI \in t_j} \epsilon^{(i)} \YI \right)
\end{align*}
Similarly, we can write the target function as: 
\[
    \nabla \bar{f}(\vec{w}) = 2(\vec{w}^{\top} - \vec{w}^{*\top}) \mathbb{E}[\vec{y} \vec{y}^{\top}] 
    \]
So the difference of the gradients is actually:
\begin{align*}
     & (\nabla f_j(\vec{w})  - \nabla \bar{f}(\vec{w})) 
     = 2(\vec{w}^{\top} - \vec{w}^{*\top}) (\sum_{\XI \in t_j} \frac{\YI \YIT}{|t_j|} - \mathbb{E}[\vec{y} \vec{y}^{\top}]) + \sum_{\XI \in t_j} \frac{\epsilon^{(i)} \YI}{|t_j|} 
\end{align*}
The first term is going to shrink as $(\vec{w}^{\top} - \vec{w}^{*\top})$ decreases. (If we draw enough data, $ \frac{1}{|t_j|} \sum_{t_j} \YI \YIT \rightarrow \Ex [\YI \YIT|t_j]$, so we'll be able to regard the other factor as a fixed ``scaling.'')  The second term approaches zero as we draw more data, $\frac{1}{|t_j|}\sum_{t_j} \epsilon^{(i)} \YI \rightarrow 0$. So given that we have drawn enough data, we will be able to bound each row of $S$ by the radius $r:=\max_{\vec{w}} \|\vec{w}\|_2$ and similarly for the whole matrix.

More concretely, if we define $S_0 :=   \|\big[ \frac{1}{|t_j|}\sum_{t_j} \YI \YIT -  \mathbb{E}[\vec{y} \vec{y}^{\top}] \big]_{\Igood} \|_{op}$,
then we find $S = \mathcal{O}(r  S_0)$. Note, $S_0$ is fixed given the data, and thus remains constant across iterations.
Furthermore, $S_0$ concentrates around $\|\big[ \Ex[ \YI \YIT|t_j] -  \mathbb{E}[\vec{y} \vec{y}^{\top}] \big]_{\Igood} \|_{op}$ and can thus be bounded. Therefore, the bound on $S$ we can guarantee will decrease when we take more points. 
We know $\frac{1}{|t_j|}\sum_{t_i}\epsilon^{(i)}$ can be bounded with a simple sub-gaussian tail bound : $\Pr[\frac{1}{|t_j|}\sum \epsilon \geq \tau ] \leq \exp[{-\frac{2\tau^2}{\sigma_{\epsilon}^2/|t_j|}}] $. 
Plugging in $\tau\gets r$, and fixing $\delta$, we find that as long as the number of examples 
$|t_j| \geq \sigma^2_{\epsilon}\log(1/\delta)/2 r^2$, then $\sum_{t_j}\epsilon^{(i)} \leq r$ with probability $1-\delta$. Taking a union bound over $\delta \gets \delta/\Mgood$, it suffices to take $|t_j| \geq\sigma^2_{\epsilon}\log(m/\delta)/2 r^2 $, 
and thus $N = \Oo (\sigma^2_{\epsilon}\log(m/\delta)/\varepsilon\mu r^2 ) )$. In summary, we obtain

\begin{lemma}\label{s} For $N=\Oo (\sigma^2_{\epsilon}\log(m/\delta)/\varepsilon\mu r^2 ) )$ example points, with probability $1-\delta$ the spectral norm of the gradients $S$ is bounded by a linear function of the radius $r:=\max_{\vec{w}} \|\vec{w}\|_2$, i.e.,
$S = \mathcal{O}(r  S_0)$.
\end{lemma}

\subsection{Analysis of Main Optimization Algorithms \ref{algo1} and \ref{UpdateWeights} } 
Let $\hat{\vec{w}}_{1:m}$ be the outputs from Algorithm \ref{algo1}. We define the weighted average parameter of terms from $\Igood$ as $\hat{\vec{w}}_{avg} := ( \sum_{i \in I_{good}} c_i |t_i| \hat{\vec{w}}_i )/ (  \sum_{i \in I_{good}} c_i |t_i|) $. 
In this section, we aim to prove a bound on $\bar{f}(\hat{\vec{w}}_{avg}) - \bar{f}(\vec{w}^*) $ by controlling the optimization error $|f_i(\hat{\vec{w}}_i) - f_i(\vec{w}^*) |$ and the statistical error $|\bar{f}(\Wavg) -f_i(\hat{\vec{w}})|$. Then we prove Algorithm \ref{UpdateWeights}  will not decrease the weight of the good terms too much.

Theorem \ref{4.1} says that Algorithm \ref{algo1} can find a small ellipse $\mathcal{E}_Y$ bounding its output, and the expected loss over $\Wavg$ is close to the expected loss of $\vec{w}^*$.

\begin{theorem}[Weighted Version of Theorem 4.1, \cite{lud}]\label{4.1}
Let $\hat{\vec{w}}_{1:m}, \hat{Y}$ be the output of Algorithm 1. Then, $\bar{f}(\hat{\vec{w}}_{avg}) - \bar{f}(\vec{w}^*) \leq 18 \frac{\Mgood \Ss r }{\sqrt{\mu}}$. Furthermore, $\hat{\vec{w}}_{avg} \in \mathcal{E}_{\hat{Y}}$ and $\tr(\hat{Y}) \leq \frac{6r^2}{\mu}$.
\end{theorem}

Lemma \ref{4.6} is a basic inequality used multiple times in the analysis. It bounds the loss via $S$. Since the algorithm is using terms instead of points, we are suffering an additional factor-$\Mgood$ blow-up of the error compared to the original bound, which is carried through the lemmas in this section.
\begin{lemma}\label{4.6}
For any $\vec{w}$ and any $\vec{w}_{1:n}$ satisfying $\vec{w}_i {\vec{w}_i}^{\top} \preceq Y$ for all i, we have 

\begin{equation} \Big| \sum_{i\in I_{good}} c_i |t_i|\langle \nabla f_i(\vec{w}) - \nabla \bar{f}(\vec{w}), \vec{w}_i \rangle \Big|
\leq \mu N \Mgood S\sqrt{\tr(Y)}.
\end{equation}
\end{lemma}

\begin{proof-of-lemma}{\ref{4.6}}
Let F be the matrix whose $i^{th}$ row is $\big(\nabla f_i(\vec{w}_0) - \nabla \bar{f}(\vec{w}_0)\big)$,
 and let W be the matrix whose $i^{th}$ row is $\vec{w}_i$. 
We consider only the rows $i \in I_{good}$, so the dimension of each matrix is $\Mgood \times d$. We have
$$\Big|\sum_{i\in I_{good}} c_i |t_i|\langle\nabla f_i(\vec{w}_0) - \nabla \bar{f}(\vec{w}_0)\rangle \Big|
= \tr(F^{\top}diag(|t_i|c_i)W)$$
$$\leq \norm{diag(|t_i|) diag(c_i) F}_{op}\norm{W}_*$$ 
 by H\H{o}lder's inequality. We can bound each part:
 \begin{align*}
     \norm{diag(t_i)}_{op}&\leq \underset{t_i \in I_{good}}{\max} |t_i| \leq \Ngood \leq \mu N \\
     \norm{diag(c)}_{op}&\leq 1 \textnormal{ since } c\in [0,1] \\
     \norm{F}_{op} &\leq \sqrt{\Mgood} S \textnormal{, by the definition of S }\\
     \norm{W}_* &\leq \sqrt{\Mgood \tr(Y)} \textnormal{, by Lemma 3.1 of \cite{lud} }
 \end{align*}
Combining these, we see that $\norm{diag(|t_i|c)F}_{op}\norm{W}_*$ is bounded by $\mu \Mgood N  S \sqrt{\tr(Y)}$.
\end{proof-of-lemma}

Lemma \ref{4.2} bounds the difference between $f_i(\hat{\vec{w}}_i) $ and $f_i(\vec{w}^*)$, based on the optimality of our solution to SDP  \ref{SDP} in Algorithm \ref{algo1}. Its proof follows identically to Lemma 4.2 of \cite{lud}.
\begin{lemma}[c.f.~Lemma 4.2 of \cite{lud}]\label{4.2}
The solution $\hat{\vec{w}}_{1:m}$ to the SDP in Algorithm \ref{algo1} satisfies: 
\begin{equation}
\sum_{i\in I_{good}} 
c_i |t_i| (f_i(\hat{\vec{w}}_i)-f_i(\vec{w}^*)) \leq \lambda \norm{\vec{w}^*}^2_2.
\end{equation}
\end{lemma}

Lemma \ref{4.3} bounds the difference between $f_i(\hat{\vec{w}}_{avg})$ and  $f_i(\hat{\vec{w}}_i) $. Its proof likewise identically follows Lemma 4.3 of \cite{lud}:
\begin{lemma}[c.f.~Lemma 4.3 of \cite{lud}]\label{4.3} 
Let $\hat{\vec{w}}_{avg} := ( \sum_{i \in I_{good}} c_i |t_i| \hat{\vec{w}}_i )/ (  \sum_{i \in I_{good}} c_i |t_i|) $. The solution $\hat{\vec{w}}_{1:m} $, $ \hat{Y} $ to Algorithm 1 satisfies
\begin{align*}
&\sum_{i\in I_{good}}  c_i |t_i| \Big( f_i(\hat{\vec{w}}_{avg})-f_i(\hat{\vec{w}}_i) \Big) 
\leq \mu N \Mgood S \Big( \sqrt{\tr(\hat{Y})} + r \Big),\\
&\sum_{i\in I_{good}}  c_i |t_i| \Big( \bar{f}_i(\hat{\vec{w}}_{avg})-\bar{f}(\vec{w}^*) \Big)
\leq \sum_{i\in I_{good}}  c_i |t_i|  \Big( f_i(\hat{\vec{w}}_{avg}) - f_i(\vec{w}^*) \Big) + 2\mu N \Mgood r S.
\end{align*}
\end{lemma}

We next consider an analogue of Lemma 4.4 of \cite{lud}.
To deal with the different weights of terms, our Algorithm \ref{UpdateWeights} considers the neighbors with their weights, and therefore uses different definitions of $a$ and $W$ (from those of \cite{lud}) in the analysis.
Lemma \ref{4.4} bounds  $\tr(Y)$ and the difference between $f_i(\tilde{\vec{w}}_i)$ and  $f_i(\hat{\vec{w}}_i) $.
\begin{lemma}\label{4.4}
For $\tilde{\vec{w}}_i$ as obtained in Algorithm \ref{UpdateWeights}, 
 $ \tilde{Y}:= \frac{2}{\mu N}\hat{W}\hat{W}^{\top}$, 
 and $W:=[\sqrt{|t_1|} \vec{w}_1, \dots ,\sqrt{|t_m|} \vec{w}_m] $.\\
 we have \[ \tilde{\vec{w}}_i\tilde{\vec{w}}_i^{\top} \preceq \tilde{Y}\] for all $i$,
 and also \[\tr(\tilde{Y}) \leq \frac{2r^2}{\mu}\]
 In addition:
\[
 \tr(\hat{Y}) \leq \frac{2r^2}{\mu} + 
 \frac{1}{\lambda}  \Big( \sum^m c_i |t_i| \big( f_i(\tilde{\vec{w}}_i) - f_i(\hat{\vec{w}}_i) \big)  \Big)
\]
\end{lemma}

\begin{proof-of-lemma}{\ref{4.4}}
Let $\tilde{\vec{w}}_i = \sum_{j=1}^m a_{ij}\hat{\vec{w}}_j$ as defined in Algorithm \ref{UpdateWeights}. 
First, we want to show $\tilde{\vec{w}}_i\tilde{\vec{w}}_i^{\top} \preceq \tilde{Y}$:
\begin{align*}
    \tilde{\vec{w}}_i\tilde{\vec{w}}_i^{\top} 
    &= \Big(\sum_{j=1}^n a_{ij}\hat{\vec{w}}_j \Big)\Big(\sum_{j=1}^n a_{ij}\hat{\vec{w}}_j \Big)^{\top}\\
    &\preceq \sum_{j=1}^n a_{ij}\hat{\vec{w}}_j \hat{\vec{w}}_j^{\top}\\
     &\preceq \sum_{j=1}^n \frac{2}{\mu N}|t_j|\hat{\vec{w}}_j \hat{\vec{w}}_j^{\top}\\
     &= \frac{2}{\mu N} \sum_{j=1}^n |t_j|\hat{\vec{w}}_j \hat{\vec{w}}_j^{\top}\\
     &= \tilde{Y}
\end{align*}
and 
\begin{align*}
\tr(\tilde{Y}) &= \frac{2}{\mu N}\tr(\hat{W}\hat{W}^{\top})\\ 
& =  \frac{2}{\mu N} \tr(diag([|t_i|])) \|\vec{w}\| \\
& \leq \frac{2}{\mu N} \sum_{i=1}^m |t_i| r^2\\
& \leq \frac{2r^2}{\mu}. 
\end{align*}
For the third claim, since $(\Wh_{1:m},\hat{Y})$ is the optimal solution of the SDP in Algorithm \ref{algo1} and $(\tilde{\vec{w}}_{1:m},\tilde{Y})$  is a feasible solution of that, we have 
\[\sum^m_{i=1}c_i |t_i|f_i(\hat{\vec{w}}_i) + \lambda \tr(\hat{Y})
\leq \sum^m_{i=1}c_i |t_i|f_i({\tilde{\vec{w}}}_i) + \lambda tr(\tilde{Y})\]
This gives us
\[
 \tr(\hat{Y}) \leq \frac{2r^2}{\mu} + 
 \frac{1}{\lambda}  \Big( \sum_{i=1}^m c_i |t_i| \big( f_i(\tilde{\vec{w}}_i) - f_i(\hat{\vec{w}}_i) \big)  \Big).
\]
\end{proof-of-lemma}

Then, analogous to Corollary 4.4 of \cite{lud}, we show $\Wavg$ can be viewed as a feasible solution to SDP in Algorithm \ref{UpdateWeights}, so we can bound $(f_i(\tilde{\vec{w}}_i) - f_i(\Wh))$ by $(f_i(\Wavg) - f_i(\Wh))$.
\begin{corollary}\label{4.4cor}
If $ \sum_{i\in I_{good}} c_i |t_i| \geq \frac{\mu N}{2} $, then
\[
    \sum_{i\in I_{good}} c_i |t_i| \big( f_i(\tilde{\vec{w}}_i) - f_i(\hat{\vec{w}}_i) \big)
    \leq \mu N \Mgood \Big( \sqrt{\tr(\hat{Y})} + r \Big)
\]
\end{corollary}

\begin{proofof}{Corollary \ref{4.4cor}}
First, we show that $\hat{\vec{w}}_{avg}$ is a feasible solution for the semidefinite program for $\tilde{\vec{w}}$ in Algorithm 2.\\
By taking $ a_{ij} = \frac{c_j |t_j|}{\sum_{j'\in I_{good}} c_{j'}|t_{j'}| }$ for $j\in I_{good}$ and $0$ otherwise, 
we get $a_{ij} \leq \frac{2|t_j|}{\mu N}$ since $\sum_{j'\in I_{good}} c_{j'}|t_{j'}| \geq \frac{\mu N}{2}$.
We see
\begin{align*}
\hat{\vec{w}}_{avg} &= \frac{\sum_{j \in I_{good}} c_j |t_j| \hat{\vec{w}}_j}{\sum_{j'\in I_{good}} c_{j'}|t_{j'}|} 
 = \sum_{j=1}^{N} a_{ij} \hat{\vec{w}}_j
\end{align*}
Then by optimality,
\[ \sum_{i\in \Igood} c_i|t_i|(f_i(\tilde{\vec{w}}_i) - f_i(\Wh)) \leq \sum_{i\in \Igood} c_i|t_i|(f_i(\Wavg) - f_i(\Wh)) \]
which is bounded by $\mu N \Mgood S \Big( \sqrt{\tr(\hat{Y})} + r \Big)$ by Lemma \ref{4.3}.
\end{proofof}

Lemma \ref{4.5} shows $\sum_{I_{good}} c_i|t_i|$ is large enough. In other words, Algorithm \ref{UpdateWeights} will not down weight good terms too much.  Its proof follows identically to Lemma 4.5 of \cite{lud}.
\begin{lemma}[c.f.~Lemma 4.5 of \cite{lud}]\label{4.5}
Suppose that $\frac{1}{N}\sum\limits_{i=1}^m c_i|t_i| (f_i( \tilde{\vec{w}}_i)-f_i(\hat{\vec{w}}_i)) \geq \frac{2}{\mu N}\sum_{i \in I_{good}} c_i|t_i| (f_i( \tilde{\vec{w}}_i)-f_i( \hat{\vec{w}}_i))$ Then, the
update step in Algorithm 2 satisfies 
\begin{equation}\frac{1}{\mu N}
\sum\limits_{i\in I_{good}}
|t_i| (c_i - c'_i) \leq \frac{1}{2N}
\sum\limits_{i=1}^m |t_i|(c_i - c'_i)
\end{equation}
Moreover, the above supposition holds if $\lambda = \frac{\sqrt{8\mu} N\Mgood S}{r}$
and $\tr(\hat{Y}) > \frac{6r^2}{2\mu}$.
\end{lemma}

Finally, we prove Theorem \ref{4.1}, which bounds the difference in the empirical loss of $\hat{\vec{w}}_{avg}$ and $\vec{w}^*$. 

\begin{proofof}{Theorem \ref{4.1}}
First, show the the weights of $\Igood$ will never be too small.
By  Lemma \ref{4.5}, the invariant $\sum_{i \in I_{good}} c_i |t_j| \geq \frac{\mu N}{2} + \frac{\mu}{2} \sum_{i=1}c_i |t_i|$ holds throughout the algorithm. Therefore we get $\sum_{i \in I_{good}} c_i |t_j| \geq \frac{\mu N}{2}  $. 
In particular, Algorithm \ref{algo1} will terminate, since Algorithm \ref{UpdateWeights} zeros out at least one outlier $c_i$ each time, and this can happen at most $m - \Mgood$ times before $\sum_{i \in I_{good}} c_i |t_i| $ would drop below $\frac{\mu N }{2}$, which we showed impossible.\\
Now, let $(\hat{\vec{w}}_{1:m}, \hat{Y})$ be the value returned by Algorithm \ref{algo1}.  By Lemma \ref{4.3} we then have
\begin{align*}
     \sum_{i \in I_{good}} c_i |t_i| (\bar{f}( \hat{\vec{w}}_{avg} ) - \bar{f}(\vec{w}^*) )
      &\leq \sum_{i \in I_{good}} c_i |t_i|  (\bar{f}( \hat{\vec{w}}_{avg} ) - \bar{f}(\vec{w}^*) )  + 2\mu N \Mgood S r\\
    & \leq \sum_{i \in I_{good}} c_i  |t_i| (\bar{f}( \hat{\vec{w}}_{i} ) - \bar{f}(\vec{w}^*) )  + 3\mu N \Mgood S  r + \sqrt{6\mu}N \Mgood S r . 
\end{align*}
By Lemma \ref{4.2}, we have $\sum_{i \in I_{good}} c_i|t_i|(f_i(\hat{\vec{w}}_i - f_i (\vec{w}^*)) \leq \lambda r^2 $ and, by the assumption we have $\tr(\hat{Y}) \leq \frac{6r^2}{\mu}$.  Plugging in $\lambda  = \frac{\sqrt{8\mu} N\Mgood S}{r}$, we get
\begin{align*}
     \sum_{i \in I_{good}} c_i |t_i| (\bar{f}( \hat{\vec{w}}_{avg} ) - \bar{f}(\vec{w}^*) )
     & \leq \lambda r^2 + 3\mu N \Mgood S r+ \sqrt{6\mu } N \Mgood S r \\
     &  = 3\mu N \Mgood S r + ( \sqrt{6} + \sqrt{8} ) \sqrt{\mu }N \Mgood S r\\
     & \leq 9 \sqrt{\mu} N \Mgood S r
\end{align*}
Since $ \sum_{i \in I_{good}} c_i |t_i|  \geq \frac{\mu N}{2}$, dividing through by $\sum_{i \in I_{good}} c_i |t_i|$ yields $\bar{f}(\hat{\vec{w}}_{avg}) - \bar{f}(\vec{w}^*) \leq 18\frac{\Mgood Sr}{\sqrt{\mu}}$.
\end{proofof}

\section{List-regression Algorithm}
We finally introduce the main algorithm to cluster the terms. Again following \citet{lud}, we initially use Algorithm \ref{algo1} to assign a candidate set of parameters $\hat{\vec{w}}_i$ for each term. 
In each iteration, we use Padded Decompositions to cluster the terms by their parameters, and then reuse Algorithm \ref{algo1} on each cluster. After each iteration, we can decrease the radius of the ellipse containing the parameters chosen by terms in $\Igood$ by half. Eventually, the algorithm will be able to constrain the parameters for all of the good terms in a very small ellipse, as illustrated in Figure \ref{pic3}.
The algorithm will then output a list of candidate parameters, with one of them approximating $\vec{w}^*$.

\begin{figure}[h]
\includegraphics[width=15cm]{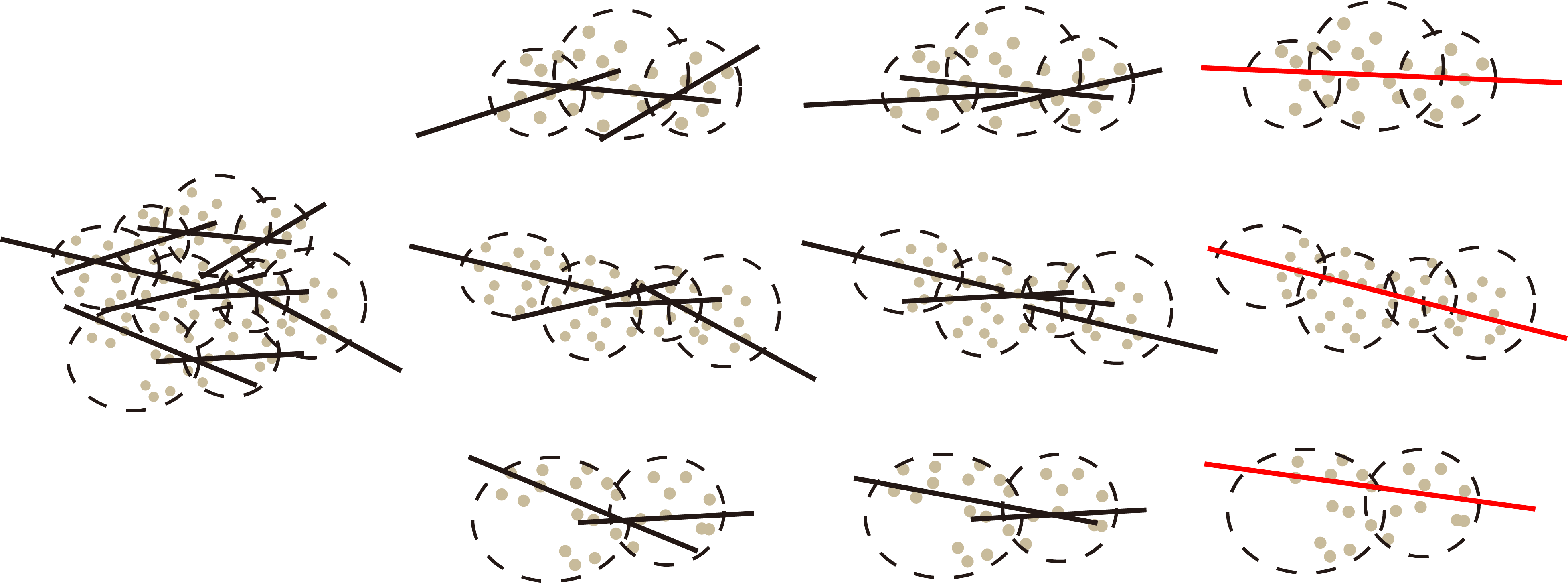}
\centering
\caption{Algorithm \ref{algo4} }\label{pic3}
\small{
\textbf{1:} Run Algorithm \ref{algo1}. Get a $\Wh_i$ for each term.
\textbf{2:} Cluster the terms by their parameter $\Wh_i$.
\textbf{3:} Iteratively re-run Algorithm \ref{algo1} on each cluster and re-cluster the terms, so that the $\Wh_i$ of $\Igood$ gradually get closer.
\textbf{4:} Finally terminate by picking a ``good enough'' cluster.
}\raggedright
\end{figure}

\subsection{Padded Decomposition}
Padded Decomposition is a randomized clustering technique developed by \citet{pad}. Given points $\{\vec{w}_1,\ldots, \vec{w}_m\}$ in a metric space,
a padded decomposition with parameters $(\rho,\tau, \delta)$ is a partitioning of the points
$\Pp := \{P_i\}$ 
satisfying the following:
\begin{compactenum}
\item Each cluster $P$ has diameter $\rho$,
\item For each point $\vec{w}_i$ and all $\vec{w}_j$ such that $\|\vec{w}_i - \vec{w}_j\| < \tau$, $\vec{w}_j$ will lie in the same cluster $P$ as $\vec{w}_i$ with probability $1-\delta$.
\end{compactenum}
Fakcharoenphol et al.~give a simple random clustering algorithm to produce padded decompositions, that uniformly samples balls with radius less than $\rho$ from the space $\mathcal{W} = \hat{\vec{w}}_{1:m}$. Intuitively, if the radius of $\Igood$, $\tau << \rho$, then we high probability, the ball with radius $\rho$ will contain all of $\Igood$.

\begin{algorithm}
\caption{Padded Decomposition}\label{algo_padded}
\begin{algorithmic}
   \STATE {\bfseries Input:} $\hat{\vec{w}}_{1:m}, \rho, \tau$.
    \STATE {\bfseries Output:} Partition $\mathcal{P} = \{T\}$.
    \STATE Initialize: let $\mathcal{P} = \emptyset$, $\mathcal{W} = \hat{\vec{w}}_{1:m}$. Sample $k \sim $ Uniform$(2,\rho)$.
    \WHILE{$\mathcal{W} \neq \emptyset$}
        \STATE Sample $i \sim $ Uniform$(1,m)$.
        \STATE Let $T \gets$ Ball$(\Wh_i, k\tau) \cap \mathcal{W}$.
        \STATE Update:  $\mathcal{P} = \mathcal{P} \cup \{T\}$. $\mathcal{W} \gets \mathcal{W} \backslash T$.
    \ENDWHILE
    \STATE {\bfseries Return:} partition $\mathcal{P}$.
\end{algorithmic}
\end{algorithm}

\begin{lemma}[Padded Decomposition]\label{padded}
If all the elements of $I$ have pairwise distance $d(\Wh_i,\Wh_j) \leq \tau$, and $\rho = \frac{1}{\delta}\tau \log(\frac{1}{\mu})$.
then for the output partition $\mathcal{P}$ of Algorithm \ref{algo_padded}, with probability least $1-\delta$, $I$ will be contained in a single cluster $T \in \mathbb{P}$.
\end{lemma}
The proof of this variant can be found in Appendix A of \cite{lud}.

\subsection{Clustering, list regression, and conditional regression}
Our algorithm uses two different kinds of clusterings of parameters to solve the list-regression problem. Given a solution to the list-regression problem, we produce a condition selecting a subset of the data (as opposed to a subset of the parameters). To reduce confusion, we now give a slightly more detailed overview of these steps, before presenting the algorithm in full detail.

In Algorithm \ref{algo4}, we will generate multiple padded decompositions in each iteration, to ensure that with high probability most of the padded decompositions preserve all of $\Igood$ in a single cluster. At the end of each iteration, we will update each term's choice of regression parameters, $   \hat{\vec{w}}_i $ by aggregating the padded decompositions.

Given a target radius $r_{final}$, we will check if the current  radius $r < r_{final}$. If so, the algorithm will greedily find a list of candidate parameters $\vec{u}_1,...,\vec{u}_s$, where the length of the list $s$ is at most $\frac{2}{\mu}$. 
We can show that one of $\vec{u}_1,...,\vec{u}_s$ must be close to $\vec{w}^*$. Finally, we will use a greedy covering algorithm following \citet{juba18} to find conditions on which the linear rules $\vec{u}_1,...,\vec{u}_s$ have low loss, and return the pair $(\hat{\vec{w}},\hat{\vec{c}})$ with at least a $\mu$ fraction of points and the smallest regression loss.

\begin{algorithm}[h]
\caption{List-regression algorithm}\label{algo4}
\begin{algorithmic}
   \STATE {\bfseries Input:} $m$ terms, target radius $r_{final}$.
    \STATE {\bfseries Output:} candidate solutions $\{ \vec{u}_1,...,\vec{u}_s \}$ and $\hat{\vec{w}}_{1:m}$.
    \STATE Initialize $r^{(1)} \gets r$, \\
    $\hat{\vec{w}}^{(1)}_{1:m} \gets $ Algorithm \ref{algo1} with origin $0$ and radius $r$ (all $i=1,\ldots,m$ are ``assigned'' an output).
\FOR{$\ell = 1,2,\ldots$}
    \STATE    $\Ww \gets \{ \hat{\vec{w}}^{(\ell)}_{i} |  \ \hat{\vec{w}}^{(\ell)}_{i}  \text{is assigned} \} $
    \IF{$r^{(\ell)} < \frac{1}{2}r_{final}$}
        \STATE Greedily find a maximal set of points $\vec{u}_1,...,\vec{u}_s$  s.t. \\
            \hspace*{0.5cm}{I}: $|B(\vec{u}_{j};2r_{\text{final}}) \cap \Ww | \geq (1 - \varepsilon)\mu N, \quad \forall j. $\\ 
            \hspace*{0.5cm}{II}: $\| \vec{u}_j - \vec{u}_{j'} \|_2 > 4r_{\text{final}}, \quad \forall j \neq j'$.
       \STATE {\bfseries Return} $\mathcal{U} = \{ \vec{u}_1,...,\vec{u}_s \}, \hat{\vec{w}}_{1:m}^{(\ell)}$.
   \ENDIF
   \FOR{$h=1$ {\bfseries to} $112\log(\frac{\ell(\ell+1)}{\delta})$}
   \STATE $\bar{\vec{w}}_{1:m}(h) \gets$ \text{unassigned} \\
    Let $\mathcal{P}_h$ be a $(\rho , 2r^{(\ell)}, \frac{7}{8})$-padded decomposition of $\Ww$ with 
    $\rho = \mathcal{O} (r^{(\ell)} log (\frac{2}{\mu}))$.
        \FOR{$T \in \mathcal{P}_h$}
            \STATE      Let $B(u,\rho) $ be a ball containing $T$. Run Algorithm \ref{algo1} on $\Hh \cap B(u,\rho) $, with radius $r = \rho$ and origin shifted to $u$.
    \STATE  for each $\Wh_i \in T$ assign $\bar{\vec{w}}_i(h)$ as the outputs of Algorithm \ref{algo1}.
    \ENDFOR
   \ENDFOR
    \FOR{$i=1$ {\bfseries to} $m$}
   \STATE Find a  $h_0$ such that $ \| \bar{\vec{w}}_i(h_0) - \bar{\vec{w}}_i(h)\|_2 \leq \frac{1}{3} r^{(\ell)} $ for at least $ \frac{1}{2} \text{\ of the \ } h\text{'s}$.
\STATE   $\hat{\vec{w}}_i^{(\ell+1)} \gets \bar{\vec{w}}_i (h_0)$ (or ``unassigned'' if no such $h_0$ exists)
   \ENDFOR
   \STATE $ r^{(\ell+1)} \gets \frac{1}{2} r^{(\ell)}$ 
\ENDFOR
\end{algorithmic}
\end{algorithm}

\subsection{Analysis of List-regression, Algorithm \ref{algo4}}
The analysis will require a ``local'' spectral norm bound that gives a tighter bound for any $\varepsilon$ fraction of points. This analysis largely follows the same outline as \cite{lud}, but differs in some key details. 
By the local spectral bound, in each iteration, we get good estimates of $\vec{w}$ for any sufficiently large subset. A key observation is that in contrast to \citet{lud}, we do not ``lose'' points from our clusters across iterations since our terms are all large enough that they are preserved. This enables a potentially arbitrarily-close approximation of $\vec{w}^*$ given enough data.

\subsubsection{Local Spectral Norm Bound}
For $\varepsilon < 1$, we define a local spectral norm bound $\Sep$ on arbitrary subsets $T$ in $\Igood$, such that $T$ takes up at least a $\varepsilon$ fraction of $\Igood$ ($ N_T \geq \varepsilon N$). Denote the number of points in $T$ by $N_T$ and the number of terms by $m_T$. We define  
\[\Sep := \max_{w\in\Hh,}  \max_{\substack{T \subset \Igood\\N_T\geq \varepsilon N}}
\frac{1}{\sqrt{ m_T }}
\| [\nabla f_j(\vec{w}) - \nabla \bar{f}(\vec{w})]_{j\in T} \|_{op}
\]
Similar to the analysis of $S$, $\Sep = \Oo(r {\Sep}_0)$, where recall, ${\Sep}_0  :=  \max_{T: N_T\geq\varepsilon N} \Big\|\Big[ \frac{1}{|t_j|}\sum_{i\in t_j} \YI \YIT -  \mathbb{E}[\vec{y} \vec{y}^{\top}\Big]_{j\in T} \Big\|_{op}$ is bounded independently of $\rt$.
We denote the value of $\Sep$ in the $\ell^{th}$ iteration by $\Set$, where $\Set = \Oo(\rt {\Sep}_0)$.

With the local spectral norm bound for the gradients, we can obtain the local version of Lemma \ref{4.3}
\begin{lemma}[c.f.~Lemma 5.2 of \cite{lud}]\label{5.2}
Let the weights $b_i \in [0,1]$ satisfy $\sum_{i \in I_{good}} b_i|t_i| \geq \varepsilon \mu N$, 
and define
 $\hat{\vec{w}}^b_{avg} :=  \sum_{i \in I_{good}} b_i|t_i| \hat{\vec{w}}_i / \sum_{i \in I_{good}} b_i|t_i| $. Then the output of Algorithm \ref{algo1}, $\hat{\vec{w}}_{1:m}, \hat{Y}$  satisfies
\begin{align*}
\sum_{i \in I_{good}} b_i|t_i| \Big(f_i(\hat{\vec{w}}^b_{avg}) - f_i(\hat{\vec{w}}_i)\Big)
 \leq  b_i|t_i| \langle \nabla f_i(\hat{\vec{w}}^b_{avg}), \hat{\vec{w}}^b_{avg} - \hat{\vec{w}}_i \rangle 
 \leq \Big( \sum_{i \in I_{good}} b_i|t_i| \Big) \Mgood S_{\varepsilon} \Big( \sqrt{tr(\hat{Y})} + r  \Big)
\end{align*}
Moreover, for any $\vec{w}, \vec{w}' \in \mathcal{H}$, we have 
\begin{align*}
\Big| \sum_{i \in I_{good}} b_i|t_i| \Big( \bar{f}(\vec{w}) - \bar{f}(\vec{w}') \Big) -
 \sum_{i \in I_{good}} b_i|t_i| \Big( f_i(\vec{w}) - f_i(\vec{w}') \Big) \Big| 
 \leq 2 \Big( \sum_{i \in I_{good}}b_i|t_i| \Big) \Mgood r S_{\varepsilon}.
\end{align*}
\end{lemma}
The proof is similar to Lemma \ref{4.3}.

\subsubsection{Proof of the Main Theorem}
We can now state and prove our main theorem for list regression. As noted at the outset, we will need to assume that the distribution over $\Igood$ is sufficiently similar relative to the degree of (strong) convexity of the loss.

\begin{theorem}\label{llr}
Let any $r_{final}$ and $\delta, \varepsilon \leq \frac{1}{2}$ be given. Suppose that the loss functions $f_i$ are $\kappa$-strongly convex and ${\Sep}_0 \leq \Oo(\frac{\kappa \sqrt{\mu}}{\Mgood\log(1/\mu)})$ for all $i\in\Igood$.
For $N = \Oo(\sigma^2_{\epsilon}\log(m/\delta)/\varepsilon\mu r_{final}^2 ) )$ example points, let $\mathcal{U}, \hat{\vec{w}}_{1:m}$ be the output of Algorithm \ref{algo4}. 
Then with probability at least $1-\delta$, 
$\mathcal{U}$ has size at most $\big \lfloor \frac{1}{(1-\varepsilon)\mu} \big \rfloor$, 
and $\min_{\vec{u} \in \mathcal{U}}\  \| \vec{u} - \vec{w}^* \|_2 \leq \mathcal{O} (r_{final})$. 
Moreover,  $\| \hat{\vec{w}}_i - \vec{w}^* \|_2 \leq \mathcal{O} (r_{final})$ for every term $i \in I_{good}$.
\end{theorem}
Towards proving Theorem~\ref{llr}, our main step is the following bound on the quality of a single iteration of Algorithm~\ref{algo4}:
\begin{theorem}\label{6.2}
For some absolute constant C, the output $\hat{\vec{w}}_{1:m}$ of Algorithm \ref{algo1} during Algorithm~\ref{algo4} satisfies 
\[\| \hat{\vec{w}}_i - \vec{w}^*\|^2_2 \leq C \cdot \frac{r^{(\ell)} \Mgood S_{\varepsilon}^{(\ell)}}{\kappa\sqrt{\mu}}\]
for all terms $i \in I_{good}$.
\end{theorem}
The key to establishing Theorem \ref{6.2} will be to use the bound on the statistical error from Lemma \ref{5.2} and the strong convexity of $f_i$. 

\begin{lemma}\label{6.3}
For any $b_i \in [0,1]$ satisfying $\sum_{i \in I_{good}} b_i|t_i|   \geq  \varepsilon \mu N $, we have 
\begin{equation}
    \frac{ \sum_{i \in I_{good}}b_i|t_i| \| \hat{\vec{w}}_i - \hat{\vec{w}}_{avg} \|^2_2 }{ \sum_{i \in I_{good}} b_i|t_i| }
    \leq \frac{2}{\kappa}(\sqrt{\tr(\hat{Y})} + r )\Mgood \Ss_\varepsilon
\end{equation}
\end{lemma}

\begin{proof-of-lemma}{\ref{6.3}}
Recall that Lemma \ref{5.2} says that for any $b_i \in [0,1] $ satisfying $\sum_{i \in I_{good}} b_i|t_i| \geq \varepsilon \mu N$, we have
\begin{align*}
\sum_{i \in I_{good}} b_i|t_i| \langle \nabla f_i (\hat{\vec{w}}^b_{avg}), \hat{\vec{w}}^b_{avg} - \hat{\vec{w}}_i \rangle
&\leq \Big( \sum_{i \in I_{good}} b_i |t_i| \Big) \Mgood S_{\varepsilon} \Big( \sqrt{\tr(\hat{Y})} + r \Big)
\end{align*}
By strong convexity of $f_i$, we have
\begin{align*}
0 &\leq \sum_{i\in I_{good}} b_i |t_i| \big( f_i(\hat{\vec{w}}^b_{avg}) - f_i(\hat{\vec{w}}_i) \big) \\
&\leq \sum_{i\in I_{good}} b_i |t_i| \Big( \langle \nabla f_i (\hat{\vec{w}}^b_{avg}),\hat{\vec{w}}^b_{avg} - \hat{\vec{w}}_i \rangle - \frac{\kappa}{2} \| \hat{\vec{w}}_i - \hat{\vec{w}}^b_{avg} \|^2_2 \Big)\\
&\leq \Bigg( \sum_{i\in I_{good}} b_i |t_i| \Bigg) \Mgood S_{\varepsilon} \Big( \sqrt{\tr(\hat{Y})}  + r \Big)      -\frac{\kappa}{2} \sum_{i\in I_{good}} b_i |t_i| \| \hat{\vec{w}}_i - \hat{\vec{w}}^b_{avg} \|^2_2.
\end{align*}
\end{proof-of-lemma}

By applying Lemma \ref{6.3} to $b'_i = \frac{1}{2} \Big ( b_i + \frac{\sum_j b_j|t_j|}{\sum_j c_j|t_j|}c_i \Big ) $, we obtain Lemma \ref{6.4}, which gives bounds in terms of $\hat{\vec{w}}_{avg}$ rather than $\hat{\vec{w}}^b_{avg}$:
\begin{lemma}\label{6.4}
For any $b_i \in [0,1]$ satisfying $\varepsilon \mu N \leq\sum_{i \in I_{good}} b_i|t_i| \leq \sum_{i \in I_{good}} c_i|t_i| $, we have 
\begin{equation}
    \frac{ \sum_{i \in I_{good}}b_i|t_i| \| \hat{\vec{w}}_i - \hat{\vec{w}}_{avg} \|^2_2 }{ \sum_{i \in I_{good}} b_i|t_i| }
    \leq \frac{10}{\kappa}(\sqrt{\tr(\hat{Y})} + r )\Mgood \Ss_\varepsilon
\end{equation}
\end{lemma}

\begin{proof-of-lemma}{\ref{6.4}}
For convenience, let us define: 
$B = \sum_{i \in \Igood} b_i|t_i|$, $ C = \sum_{i \in \Igood} c_i|t_i|$, 
$b'_i = \frac{1}{2} \Big ( b_i + \frac{\sum_j b_j|t_j|}{\sum_j c_j|t_j|}c_i \Big ) = \frac{1}{2} b_i + \frac{1}{2} \frac{B}{C}c_i$. 
Notice that $\sum_{\Igood} b_i'|t_i| = B$ and $\Wavg^{b'} = \frac{1}{2}\Wavg^{b} + \frac{1}{2}\Wavg$.\\

We invoke Lemma \ref{6.3} twice, on $b$ and $b'$ respectively:
\begin{align*}
    \frac{1}{B}\sum_{i\in \Igood} b_i' |t_i|\|\Wh_i-\frac{1}{2}\Wavg^{b} - \frac{1}{2}\Wavg\| =
    \frac{1}{B}\sum_{i\in \Igood} b_i'|t_i|\|\frac{1}{2}\Wh_i - \frac{1}{2}\Wavg^{b'}\|
    &\leq  \frac{2}{\kappa}(\sqrt{\tr(\hat{Y})} + r ) \Mgood\Ss_\varepsilon\\
    \frac{1}{B}\sum_{i\in \Igood} b_i |t_i|\|\frac{1}{2}\Wh_i -\frac{1}{2}\Wavg^{b}\| &\leq  \frac{1}{\kappa}(\sqrt{\tr(\hat{Y})} + r ) \Mgood\Ss_\varepsilon.
\end{align*}
Since for any $i$, $b_i' \leq \frac{1}{2}b_i$
\begin{align*}
     \frac{1}{B}\sum_{i\in \Igood} b_i |t_i|\|\Wh_i-\frac{1}{2}\Wavg^{b} - \frac{1}{2}\Wavg\|
    \leq 
    \frac{2}{B}\sum_{i\in \Igood} b_i' |t_i|\|\Wh_i-\frac{1}{2}\Wavg^{b} - \frac{1}{2}\Wavg\|
    \leq 
    \frac{4}{\kappa}(\sqrt{\tr(\hat{Y})} + r ) \Mgood\Ss_\varepsilon
\end{align*}
Combining the two inequalities
\begin{align*}
    \frac{1}{B}\sum_{i\in \Igood} b_i' |t_i|\|\Wh_i-\Wavg^{b} \|
    &\leq   \frac{1}{B}\sum_{i\in \Igood} b_i' |t_i| 2\Big ( \|\Wh_i-\frac{1}{2}\Wavg^{b} - \frac{1}{2}\Wavg\| +
 \|-\frac{1}{2}\Wh_i + \frac{1}{2}\Wavg^{b}\| \Big)\\
    & \leq 2\Big( \frac{4}{\kappa}(\sqrt{\tr(\hat{Y})} + r ) \Mgood\Ss_\varepsilon + \frac{1}{\kappa}(\sqrt{\tr(\hat{Y})} + r ) \Mgood\Ss_\varepsilon\Big)\\
    & = \frac{10}{\kappa}(\sqrt{\tr(\hat{Y})} + r )\Mgood \Ss_\varepsilon
\end{align*}
\end{proof-of-lemma}

\begin{corollary}\label{6.4cor}
In particular, no set of terms comprising more than a $\varepsilon \mu $ fraction of the data (or probability weight) can have 
$\| \hat{\vec{w}}_i - \hat{\vec{w}}_{avg} \|_2^2 > 
\frac{10}{\kappa}(\sqrt{\tr(\hat{Y})} + r )\Mgood \Ss_\varepsilon$.
\end{corollary}

\begin{proof}
Consider terms $t_1, \dots, t_q$ with 
$\Pr[t_1 \vee \dots \vee t_q] > \varepsilon\mu$.
Assume for contradiction that for all of these terms, 
$\| \hat{\vec{w}}_i - \hat{\vec{w}}_{avg} \|_2^2 > 
\frac{10}{\kappa}(\sqrt{\tr(\hat{Y})} + r ) \Mgood \Ss_\varepsilon$.
We can assign $b_i$ for each $t_i$ such that $\sum b_i = \varepsilon \mu$. Then
\begin{align*}
    \frac{ \sum_{i \in I_{good}}b_i|t_i| \| \hat{\vec{w}}_i - \hat{\vec{w}}_{avg} \|^2_2 }{ \sum_{i \in I_{good}} b_i|t_i| }
    &> \frac{ \sum_{i \in I_{good}}b_i|t_i| 
   \frac{10}{\kappa}(\sqrt{\tr(\hat{Y})} + r )\Mgood \Ss_\varepsilon}{ \sum_{i \in I_{good} } b_i|t_i|} \\
    &= \frac{10}{\kappa}(\sqrt{\tr(\hat{Y})} + r )\Mgood \Ss_\varepsilon
\end{align*}
which contradicts Lemma \ref{6.4}.
\end{proof}

\paragraph{Key Observation}
As we deleted all terms of size smaller than $\varepsilon \mu N$,
all the remaining terms have at least $\varepsilon \mu $ probability-weight (or $\varepsilon \mu N$ empirical size).
Therefore, every term will satisfy 
\[\| \hat{\vec{w}}_i - \hat{\vec{w}}_{avg} \|_2^2 \leq 
\frac{10}{\kappa}(\sqrt{\tr(\hat{Y})} + r )\Mgood \Ss_\varepsilon.\]

We can subsequently obtain Theorem \ref{6.2} by thus invoking Corollary \ref{6.4cor}:

\begin{proofof}{Theorem \ref{6.2}}
By Corollary \ref{6.4cor}, we have 
$\| \hat{\vec{w}}_i - \hat{\vec{w}}_{avg} \|_2^2 \leq 
\frac{10}{\kappa}(\sqrt{\tr(\hat{Y})} + r )\Mgood \Ss_\varepsilon$
for all $i \in I_{good}$.
$\tr(\hat{Y}) \leq \Oo(\frac{r^2}{\mu})$,
so $\| \hat{\vec{w}}_i - \hat{\vec{w}}_{avg} \|_2^2 \leq 
\frac{10}{\kappa}(\frac{r}{\sqrt{\mu}} + r )\Mgood \Ss_\varepsilon = \Oo(\frac{r \Mgood \Ss_\varepsilon}{\kappa \sqrt{\mu}})$.
In addition, by Theorem \ref{4.1}, 
$\bar{f}(\hat{\vec{w}}_{avg}) - \bar{f}(\vec{w}^*) \leq \Oo( \frac{r \Mgood \Ss_\varepsilon}{\sqrt{\mu}})$. 
By the strong convexity of $\bar{f}$, $\|\hat{\vec{w}}_{avg} - \vec{w}^*\|_2^2 \leq \Oo (\frac{r\Mgood \Ss_\varepsilon}{\sqrt{\mu}})$. 
We combine the bounds to obtain 
$
\|\hat{\vec{w}}_{i} - \vec{w}^*\|_2^2 
\leq 2(\|\hat{\vec{w}}_{i} - \hat{\vec{w}}_{avg}\|_2^2 + \|\hat{\vec{w}}_{avg} - \vec{w}^*\|_2^2)
\leq \mathcal{O}\Big(\frac{r \Mgood \Ss_\varepsilon}{\sqrt{\mu}}\Big)
$
\end{proofof}

Finally, using Theorem \ref{6.2}, we show the radius $\rt$ (used in the $\ell^{th}$ iteration) can be decreased by half at each iteration.
\begin{lemma}\label{6.5}
In Algorithm \ref{algo4}, denote the set of parameters of good points of $\ell^{th}$ iteration by
$I^{(\ell)}_{good} := \{\hat{\vec{w}}^{(\ell)}_i : i \in I_{good}\}.$
If $\|\hat{\vec{w}}^{(\ell)}_i - \vec{w}^*\|_2 \leq \rt$ and ${\Sep}_0 \leq C'\cdot \frac{\kappa \sqrt{\mu}}{\Mgood\log(2/\mu)}$ for some constant $C'$,
then with probability $(1-\frac{\delta}{\ell(\ell+1)})$ over the randomly chosen padded decompositions,
 $\|\hat{\vec{w}}^{(\ell+1)}_i - \vec{w}^*\|_2 \leq \frac{1}{2}\rt$.
\end{lemma}

\begin{proof}
We call a padded decomposition partition $\Pp_h$ {\em good} if all of the terms of $\Igoodt$ lie in a single cluster of $\Pp_h$. Denote the set of padded decompositions where $\Pp_h$ is good by $H$.

In the algorithm, we draw $q = 112\log{\frac{\ell(\ell+1)}{\delta}}$ random padded decompositions with parameters $(\rho, 2\rt, \frac{1}{8}) $, where $\rho = \Oo(\rt\log{\frac{2}{\mu}})$,
so that (i) each cluster $P$ of $\Pp_h$ has diameter at most $\Oo(\rt\log{\frac{2}{\mu}})$, and (ii) for each padded decomposition and a parameter vector, all other parameter vectors within $2 \rt$ will lie in the same cluster with probability 7/8.

Since we assume $\|\hat{\vec{w}}^{(\ell)}_i - \vec{w}^*\|_2 \leq \rt$, with probability $\frac{7}{8}$, all of $\Igoodt$ will lie in a single cluster, i.e. this padded decomposition $\Pp_h$ is good.
Then, by a Chernoff Bound, the total number of good padded decompositions will be larger than $\frac{3}{4}q$ with probability $1-\frac{\delta}{\ell(\ell+1)}$.

For a good padded decomposition ($P$ is the cluster containing the terms of $\Igoodt$), $ \vec{w}^*$ is within distance $\rt$ of $P$.
Therefore, if $P\subset B(u,\rho)$, then $\vec{w}^* \in B(u,\rho+\rt)$. As we run Algorithm 1 on $B(u,\rho+\rt) $, Theorem \ref{6.2} will give us:
\begin{align*}
\| \bar{\vec{w}}_i(h)  - \vec{w}^*\|^2_2 &\leq C \cdot \frac{ \Mgood rS_{\varepsilon}}{\kappa\sqrt{\mu}}\\
\| \bar{\vec{w}}_i(h)  - \vec{w}^*\|_2 &\leq \mathcal{O}  \Big( \sqrt{ \frac{\Mgood (\rho+\rt)\Set }{\kappa \sqrt{\mu}} }\Big)\\
&= \mathcal{O}  \Big( \sqrt{ \frac{\Mgood \rt \log{\frac{2}{\mu}} \ \Set }{\kappa \sqrt{\mu}} }\Big)
\end{align*}
where $\bar{\vec{w}}_i(h) $ is the output $ \hat{\vec{w}}_i$ of Algorithm 1.

We want to show $\| \bar{\vec{w}}_i(h)  - \vec{w}^*\|_2 \leq \frac{1}{6}\rt  $. Recall that 
$ \Set \leq \mathcal{O} ({\Sep}_0\rt)$, so it suffices to have

\[\rt\cdot \sqrt{ \frac{\Mgood \log{\frac{2}{\mu}} \ {\Sep}_0 }{\kappa \sqrt{\mu}} }  \leq \mathcal{O} ( \rt), \]
i.e., ${\Sep}_0\leq \mathcal{O}\Big( \frac{\kappa\sqrt{\mu}}{\Mgood\log\frac{2}{\mu}}\Big)$, which is true by hypothesis (for some suitable $C'$).

Since $\| \bar{\vec{w}}_i(h)  - \vec{w}^*\|_2 \leq \frac{1}{6}\rt  $
for any two good iterations $h$ and $h'$, and each iteration is only bad with probability $\leq \frac{1}{4}$,
we can pick $h_0$ such that $\| \bar{\vec{w}}_i(h_0)  - \bar{\vec{w}}_i(h)\|_2 \leq \frac{1}{3} \rt$ is true for half of the iterations.
For any good $h$,
\begin{align*}
    \| \bar{\vec{w}}_i(h_0)  - \vec{w}^*\|_2 
    &\leq \| \bar{\vec{w}}_i(h_0)  - \bar{\vec{w}}_i(h)\|_2 +  \| \bar{\vec{w}}_i(h)  - \vec{w}^*\|_2 \\
    &\leq \frac{1}{3} \rt + \frac{1}{6}\rt \\
    &\leq \frac{1}{2}\rt
\end{align*}
That is, $\hat{\vec{w}}_i^{(\ell+1)}:= \bar{\vec{w}}_i(h_0)$ is within $\frac{1}{2}\rt$ of $\vec{w}^*$.
\end{proof}

\begin{proofof}{Theorem \ref{llr}}
Lemma \ref{6.5} shows that on the $\ell$th iteration, $\vec{w}^*$ lies in a ball of radius $\rt$ around each $\hat{\vec{w}}_i^{(\ell)}$ for $i\in\Igood$, where $\rt$ decreases by half on each iteration.
In the final iteration, when $\rt$ reaches the target accuracy radius $r_{final}$, the algorithm greedily finds disjoint balls $B(\vec{u}_{j};2r_{\text{final}})$ on the parameter space $\mathcal{W}$, such that the corresponding terms for the covered parameters contain at least $(1 - \varepsilon)\mu N$ points, i.e for each ball $B$,
\[\sum_{\hat{\vec{w}}_i\in B } |t_i| \leq (1 - \varepsilon)\mu N\]
Since for all  $i \in I_{good}$, $\|\Wh - \Wstar\| < r_{final}$, we now argue that
all of the terms in $\Igood$ will lie in one ball. Indeed, if no term of $\Igood$ is contained in any ball, then any term of $\Igood$ gives a candidate for the greedy algorithm to add to the list. Therefore, at least  one of the good terms $\Wh_i, i\in\Igood$ must be contained in some ball. Then for this ball $B(\vec{u},r_{final})$,
\begin{align*}
    \|\vec{u}-\Wstar\| 
    &\leq \|\vec{u}-\Wh_i\| + \|\Wh_i -\Wstar\| \\
    &\leq 2 r_{final} + r_{final}\\
    & = \mathcal{O}( r_{final})
\end{align*}
Since each ball contains at least $(1-\varepsilon)\mu$ points, there can be at most $\big \lfloor \frac{1}{(1-\varepsilon)\mu} \big \rfloor$ such balls, which completes the proof.
\end{proofof}

\section{Obtaining a $k$-DNF Condition}
Once we get outputs $\{ \vec{u}_1,...,\vec{u}_s \}$ from Algorithm \ref{algo4},
we switch from the parameter space $\{\vec{w}\}$ back to the Boolean data space $\{\vec{x}\}$, to search for corresponding conditions $\vec{c}$ for each candidate parameter $\vec{u}_i$. If we find a pair $(\vec{u},\vec{c})$ such that $\vec{c}$ contains enough points and the loss $f_\vec{c}(\vec{u})$ is small, we return this pair as the final solution.

Suppose $\vec{u}$ is one of the candidates such that $\|\vec{u}-\Wstar\| < \mathcal{O}(r_{final}) =: \gamma $, 
then $|\bar{f}(\vec{u})-\bar{f}(\Wstar) | \leq \gamma L = \mathcal{O}(\gamma)$, for some Lipschitz constant $L$ (since $f$ is just a regression loss on a bounded space, it is Lipschitz continuous). 
Recalling $\bar{f}$ is nonnegative, if $\bar{f}(\Wstar) \leq \epsilon$, 
then $\bar{f}(\vec{u}) \leq \epsilon + \Oo(\gamma)$. The above also holds for the empirical approximation to $\bar{f}$, $f_{\Igood}$.

\subsection{Bounding the Double-Counting Effect}
We now address the effect of our double-counting of points. Recall that we introduced a copy of a point for each term it satisfied. 
Observe that on $\Igood$, which contains $\Mgood$ terms, this is at most $\Mgood$ copies. We thus obtain
\begin{lemma}\label{dcbound}
Let $\vec{u}$ be such that $\|\vec{u}-\Wstar\| < \gamma$. Then
$f_{\Igood}(\vec{u})\leq \frac{1}{|\Igood|}\sum_{t_j\in\vec{c}^*}|t_j|f_j(\vec{u}) +\Oo(\Mgood\gamma) \leq \Mgood \epsilon + \Oo(\Mgood \gamma)$ (where $f_{\Igood}$ refers to the true empirical loss, without duplicated points).
\end{lemma}
\begin{proof}
Assume $\Wtrue$ is the true optimal linear fit: ignoring the common $1/|\Igood|$ scaling,
\[\Wtrue := \underset{\vec{w}}{\text{argmin}} f_{\Igood}(\vec{w}) = \underset{\vec{w}}{\text{argmin}} \sum_{i\in \Igood} \FI(\vec{w})\]
and $\Wstar$ is the optimal linear fit for the double counted data: letting $a^{(i)} \in [1, \Mgood]$ denote the number of copies of each point after duplication and again ignoring the $\sum_{i\in\Igood}a^{(i)}$ scaling factor,
\[\Wstar := \underset{\vec{w}}{\text{argmin}} \sum_{i\in \Igood} a^{(i)} \FI(\vec{w}).\]
Now, for any $\vec{w}$, observe that $ \sum_{i\in \Igood} \FI(\vec{w}) \leq  \sum_{i\in \Igood} a^{(i)} \FI(\vec{w})$ and
\begin{align*}
    \sum_{i\in \Igood} a^{(i)} \FI(\vec{w})
    \leq (\max_{i\in\Igood} a^{(i)}) \sum_{i\in \Igood}  \FI(\vec{w})
    \leq \Mgood \sum_{i\in\Igood}  \FI(\vec{w}).
\end{align*}
Therefore,
\[
|\Igood|f_{\Igood}(\Wstar)\leq \sum_{i\in \Igood} a^{(i)} \FI(\Wstar) \leq \sum_{i\in \Igood} a^{(i)} \FI(\Wstar_{true}) \leq \Mgood|\Igood| f_{\Igood}(\Wstar_{true})
\]
where we note that for any $\vec{w}$, $\sum_{i\in \Igood}a^{(i)}f^{(i)}(\vec{w})=\sum_{t_j\in\vec{c}^*}|t_j|f_j(\vec{w})$.
Therefore, if $f_{\Igood}(\Wtrue) \leq \epsilon + \Oo(\gamma)$, then 
\[
f_{\Igood}(\Wstar)\leq \frac{1}{|\Igood|}\sum_{t_j\in\vec{c}^*}|t_j|f_j(\Wstar) \leq \Mgood \epsilon + \Oo(\Mgood\gamma)
\]
where $|f_{\Igood}(\Wstar)-f_{\Igood}(\vec{u})|\leq \Oo(\gamma)$ and $|\sum_{t_j\in\vec{c}^*}\frac{|t_j|}{|\Igood|}f_j(\Wstar)-\sum_{t_j\in\vec{c}^*}\frac{|t_j|}{|\Igood|}f_j(\vec{u})|\leq \Oo(\Mgood\gamma)$.
\end{proof}

\subsection{Greedy Set-Cover}
We have obtained a parameter vector $\vec{u}$ such that the loss for each term $f_i(\vec{u})$ is close to $f_i(\Wstar)$.
We can now use a greedy weighted partial set-cover algorithm for the {\em ratio objective} (optimizing the ratio of the cost to the number of elements covered), e.g., as presented and analyzed by \citet{zmj17} following \citet{slavik97}, to find the corresponding conditions $\vec{c}$. At a high level, given regression parameters $\vec{u}$, we compute the loss $f(\vec{u})$ for each point, and then use the covering algorithm to find a collection of terms that cover enough points while minimizing the loss.
\begin{algorithm}[htb]
        \caption{Partial Greedy Algorithm}\label{greedy-alg}
        \textbf{Input:} finite set $\mathcal{T}=\{T_1,...,T_m\}$, costs 
$\{\omega_1,...,\omega_m\}, \mu,\gamma \in (0,1]$ \\
        \textbf{Output:} Condition $\hat{\vec{c}}$
        \begin{algorithmic}[1]
        \STATE Initialize $\hat{\vec{c}} = \emptyset$\\
        \WHILE{$(1-(2/3)\gamma)\mu N >  \left|\bigcup_{T_j\in \hat{\vec{c}}}T_j\right|$}
        \STATE Choose the first $T_j\in\mathcal{T}\setminus\hat{\vec{c}}$ covering at least $\frac{\mu\gamma}{3\Mgood}N$ additional examples, that minimizes $\omega_j/|T_j|$, for $T_j \in \mathcal{T}\setminus\hat{\vec{c}}$.\\ 
        \STATE Add $T_j$ to $\hat{\vec{c}}$, set each other $T_{j'}=T_{j'}\setminus T_j$
        \ENDWHILE
        \STATE Return $\hat{\vec{c}}$
        \end{algorithmic}
        \label{alg:Insert}
\end{algorithm} 

Specifically, we use Algorithm~\ref{greedy-alg}, associating the
term $t_j$ with the set $T_j=\{\vec{x}^{(i)}:t_j(\vec{x}^{(i)})=1\}$ and with cost
$\omega_j=|t_j|f_j(\vec{u})$.
In other words, it associates with $t_j$ the set of examples such that $t_j(\vec{x}^{(i)})=1$, and 
assigns each set the cost $\sum_{\vec{x}^{(i)}:t_j(\vec{x}^{(i)})=1}f^{(i)}(\vec{u})$.
It then follows the standard greedy algorithm for weighted partial set cover on this instance, modified slightly to ignore sets that cover too few examples. This latter condition will allow us to control the size of the formula $\hat{\vec{c}}$ we find as a function of $\Mgood$ and $\gamma$. Following \citet{juba18}, we will be able to leverage this bound on the size of $\hat{\vec{c}}$ to obtain a better approximation ratio for the loss.

\begin{lemma} \label{3.4.2}
Given a set of $N=\Oo(\frac{1}{\mu\gamma^2}(\frac{t}{\gamma}+\frac{\sigma^2L^2}{\epsilon})\log\frac{m}{\delta})$ points $\{\vec{x}^{(i)}\}_{i=1}^N$ with weights $f^{(i)}(\vec{u})$ and terms $\{t_j\}_{j=1}^m$, if there exists a $\Mgood$-term $k$-DNF $\vec{c}^*$ that is satisfied by a $\mu$-fraction of the points with  total loss $\epsilon$, then the weighted greedy set cover algorithm (Algorithm~\ref{greedy-alg}) finds a $3\Mgood/\gamma$-term $k$-DNF $\hat{\vec{c}}$, that is satisfied by a $(1-\gamma)\mu$-fraction of the points with total loss $\Oo( \Mgood\log(\mu N) \epsilon)$
\end{lemma}
\begin{proof}
Observe that since the loss is non-negative, for any $\vec{w}$ the terms $t_j$ in a formula $\vec{c}$ must satisfy 
\[
\Ex[f(\vec{w})|\vec{c}]\Pr[\vec{c}(\vec{x})=1]\leq \sum_{t_j\in\vec{c}}\Ex[f(\vec{w})|t_j]\Pr[t_j(\vec{x})=1]
\]
as the RHS simply counts the contribution of some points multiple times, depending on how many terms of $\vec{c}$ it satisfies.
The latter quantities are approximated by $\frac{|t_j|}{N}f_j(\vec{w})$.
By taking $N$ sufficiently large, we can ensure that with high probability, $\frac{1}{N}\sum_{t_j\in\vec{c}^*}|t_j|f_j(\vec{u})$ is at most 
$\Mgood\mu\epsilon+\Oo(\Mgood\mu\gamma)$ by Lemma~\ref{dcbound}. 
We note that the cost of any cover $\vec{c}$ is $\sum_{t_j\in\vec{c}}|t_j|f_j(\vec{u})$, which is the (double-counted) loss of $\vec{c}$ (up to rescaling). 

It follows from an analysis by Haussler (\citeyear{haussler88}) that since there are at most $\sim m^{3\Mgood/\gamma}$ formulas consisting of at most $\frac{3\Mgood}{\gamma}$ terms, $\Oo(\frac{\Mgood}{\mu\gamma^3}\log\frac{m}{\delta})$ examples suffice to guarantee that any $\frac{3}{\gamma}\Mgood$-term formula (out of the $m$ possible terms) that empirically satisfies at least $(1-(2/3)\gamma)\mu N$ examples must be satisfied with probability at least $(1-\gamma)\mu$ overall, and conversely, any formula (in particular, $\vec{c}^*$) that is true with probability at least $\mu$ will empirically satisfy at least $(1-\gamma/3)\mu N$ examples. 
Moreover, we can obtain more generally that $\Oo(\frac{\sigma^2L^2}{\mu\epsilon\gamma^2}\log\frac{m}{\delta}))$ examples suffice to guarantee that the loss on all terms is estimated to within a $(1\pm\gamma)$-factor.
The above will simultaneously hold with probability $1-\delta$ for suitable constants.

\citet{zmj17} showed that the greedy algorithm obtains a $3H((1-(2/3)\gamma)\mu N)$-approximation to the
minimum weight set cover under the ratio objective, where $H(\ell)$ denotes the $\ell$th harmonic number, which is $\leq \log(\mu N)+1$.
We have modified the algorithm slightly, to ignore sets that cover fewer than $\frac{\mu\gamma}{3\Mgood}N$ points. Observe that if all $\Mgood$ terms of $\vec{c}^*$ fail this condition, then in total there must be at most $\Mgood\cdot \frac{\mu\gamma}{3\Mgood}N =
\frac{\mu\gamma}{3}N$ points out of the $(1-\gamma/3)\mu N$ points of $\vec{c}^*$ uncovered, so at least
$(1-(2/3)\gamma)\mu N$ points are already covered by $\hat{\vec{c}}$, and thus the algorithm would already terminate.
We can thus still compare the ratio of the set chosen by the greedy algorithm to the term of $\vec{c}^*$ that still picks up sufficiently many additional points, with the smallest ratio.
Thus, the rest of the argument remains the same and the greedy algorithm finds a formula
$\hat{\vec{c}}$ with 
\[
\frac{\sum_{t_j\in\hat{\vec{c}}}|t_j|f_j(\vec{u})}{|\{\vec{x}^{(i)}:\hat{\vec{c}}(\vec{x}^{(i)})=1\}|}\leq 
3H(\mu N)\frac{\sum_{t_j\in\vec{c}^*}|t_j|f_j(\vec{u})}{|\{\vec{x}^{(i)}:\vec{c}^*(\vec{x}^{(i)})=1\}|}\leq
\Oo( \Mgood (\epsilon+\gamma) H(\mu N)).
\]
Since every iteration covers at least $\frac{\mu\gamma}{3\Mgood}N$ additional points, $\hat{\vec{c}}$ has at most $\frac{3}{\gamma}\Mgood$ terms. Thus the above bounds for small formulas ensure that $\hat{\vec{c}}$ indeed has $\Pr[\hat{\vec{c}}(\vec{x})=1]\geq (1-\gamma)\mu$ and, by our first observation, the loss of $\hat{\vec{c}}$ is indeed at most $\Oo(\Mgood (\epsilon+\gamma) H(\mu N))$ as claimed.
\end{proof}

\subsection{Generalization Bound}
Theorem \ref{llr} guarantees that Algorithm~\ref{algo4} finds a parameter vector approximating the optimal empirical loss, and Lemma \ref{3.4.2} gives us the sample complexity needed to find a good condition for a fixed linear predictor. However, there is still a gap between the empirical loss and true loss. In this section, we will bound the generalization error of linear regression on each possible $k$-DNF, and then take a union bound to achieve the main theorem. In short, the process will blow-up the complexity by $d^3$, where $d$ is the dimension of the feature space.

We will use the Rademacher generalization bound for linear predictors. For a set of data,  Lemma \ref{rcgen} bounds the gap between the expected loss $L_p(\cdot)$ and the empirical loss $\hat{L}_p(\cdot)$:
\begin{lemma}[\cite{bartlett2002rademacher}, \cite{kakade2009complexity}]\label{rcgen}
For $b>0$, $p\geq 1$, random variables $(\vec{Y},Z)$ distributed over
$\{\vec{y}\in\R^d:\|y\|_2\leq b\}\times [b,b]$, and any $\delta\in
(0,1)$, let $L_p(\vec{w})$ denote $\mathbb{E}[|\langle\vec{w},\vec{Y}
\rangle-Z|^p]$, and for an an i.i.d.~sample of size $N$ let
$\hat{L}_p(\vec{w})$ be the empirical loss $\frac{1}{N}\sum_{j=1}^N
|\langle \vec{w},\vec{y}^{(j)}\rangle-z^{(j)}|^p.$
We then have that with probability $1-\delta$ for all
$\vec{w}$ with $\|w\|_2\leq b$,
\[
|L_p(\vec{w})- \hat{L}_p(\vec{w})|\leq\frac{2pb^{p+1}}{\sqrt{N}}+b^p\sqrt{\frac{2\ln(4/\delta)}{N}}.
\]
\end{lemma}
In our case, we only consider squared error; in other words, $p=2$ for us. And notice, in our setting, we are given a bound $B$ on the magnitude of the entries, so $b \leq \sqrt{d}B$. Equivalently, we get
\[ |L_p(\vec{w})- \hat{L}_p(\vec{w})|\leq\frac{4 B^3 d^{\frac{3}{2}}}{\sqrt{N}}+o(B^2d).\]
Therefore, to bound the gap of the expected loss $L_p(\cdot)$ and the empirical loss $\hat{L}_p(\cdot)$, it suffices for our sample complexity $N$ to grow with $B^6d^3$.

The above lemma bounds the gap of a specific (conditional) distribution and set of data. To obtain a bound on  the loss obtained by our algorithm in general, we can simply take a union bound over the conditional distributions given by all $t$-term $k$-DNFs. Since $x$ has $n$ attributes, there are ${n\choose k}$ terms, which is at most $m = n^k$. And, there are ${m\choose t}$ $t$-term $k$-DNFs, so in total we have $\Oo(n^{kt})$ such $k$-DNFs, which means it will suffice to replace $\delta$ with $\frac{\delta}{n^{kt}}$ before taking the union bound. 
Actually, recalling that Lemma~\ref{3.4.2} only guarantees that our algorithm produces $3t/\gamma$-term $k$-DNFs as output,
overall we thus achieve a $\Oo( \Mgood\log(\mu N)(\gamma+\epsilon) )$ approximation as claimed in the main theorem with $N = \Oo(\frac{B^6 d^3\sigma^2L^2t}{\mu\gamma^3}\log(\frac{m}{\delta /m^{t/\gamma}}) ) ) = \Oo(\frac{B^6 d^3\sigma^2L^2t^2}{\mu\gamma^4}\log(m/\delta) ) )$ examples.

\section{Experiments}
We now present several experiments. We present three synthetic data experiments with planted solutions to illustrate our algorithm's capabilities.
For these synthetic data experiments, we use our algorithm to generate a list of candidate parameters and their corresponding DNF. If there is one pair that is close to our planted solution (or another output with even lower error), then we view the task as successful. 
We then present experiments showing that on real data sets, our algorithm can obtain loss that is consistently similar to or significantly smaller than that of the sparse $\ell_2$ regression algorithm \citep{hjlw19}. 
The code is written in Matlab with the optimization toolbox Yalmip~\cite{Lofberg2004}.
Our code can be found at {\it https://github.com/wumming/lud}.

\subsection{Toy Examples}
We first present a couple of toy examples that provide nice visual illustrations of our task.

\begin{figure}[h]
\includegraphics[height=7cm]{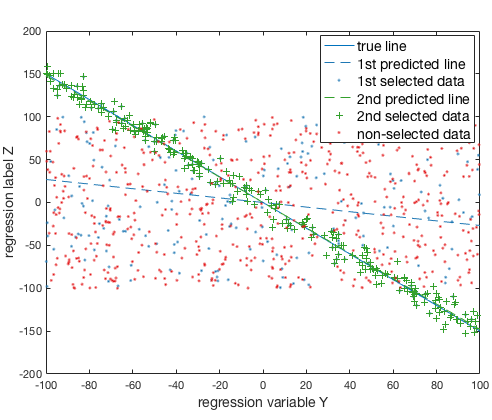}
\centering
\caption{ Line with Uniform}\label{experiment1}
\small{
Conditional linear regression with $d = 1$ in the $y \times z$ plane: For the bad data, $y, z \in [-100,100]$, while for good data, $y \in [-100,100]$ and $z = -1.5 y + noise$, the $dim(\vec{x}) = 6$. The algorithm gives two pairs of candidates. The blue one approximates the uniformly generated bad data and the green one catches the planted good data.
}\raggedright
\end{figure}

\paragraph{Example 1: Line with Uniform}
In this experiment, we first choose a $4$-term $2$-DNF at random. We uniformly generate Boolean attributes serving as $\vec{x}$, where $\mu N$ of them satisfy the chosen DNF (good data) and the other $(1-\mu)N$ don't (bad data). Then the $\vec{y}$ parts are all uniformly generated real attributes.
We also generate a target linear rule $\vec{w}^*$  with dimension equal to that of $\vec{y}$. For the good data, we set their labels $\ZI = \langle \YI,\vec{w}^* \rangle + noise$, where the noise is independently generated from zero-mean Gaussian distribution. For the bad data, $\ZI$ is independently generated from a uniform distribution similar to $\vec{y}$.
We set the dimension of $\vec{y}$ simply to be $1$, as an illustration on the $y \times z$ plane. There are $N = 1000$ points in total and $\mu = 0.25$. For the bad data $y^{(i)},\ZI \in [-100,100]$, while for good data $y^{(i)} \in [-100,100]$ and $\ZI = -1.5 y^{(i)} + noise$. The dimension of $\vec{x}$ is set to $6$. Consequently there are $m = 72$ terms in total. As shown in  Figure \ref{experiment1}, the algorithm finds two pairs of potential linear predictors and DNFs, of which the green one overlaps with the chosen planted rule, and the blue one matches the uniformly generated bad points.

\paragraph{Example 2: Sine function}
In this experiment, we didn't specify good and bad data, but uniformly generated $y^{(i)} \in [-\pi,\pi]$ and $z^{(i)} = \sin(y^{(i)}) + noise$. We attach a constant coordinate to $y$ to represent the intercept of the line. Define $x_1, x_2$ and $x_3$ to be ``$y \geq -\pi/2$'', ``$ y \geq 0 $'' and ``$y \geq \pi/2$'' respectively and set $\mu = 0.5$ and $S = 0.01$. That means, the algorithm is asked to find a DNF segment that contains at least $50\%$ of points that can fit a line. The algorithm outputs three pairs as shown on the Figure \ref{experiment2}. The blue cluster $x_1 \wedge \neg x_3$  has the smallest error, which is the points in interval $[-\pi/2, \pi/2]$.

\begin{figure}[h]
\includegraphics[height=7cm]{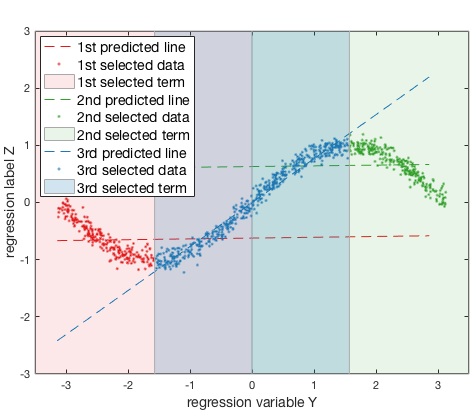}
\centering
\caption{Sine Function}\label{experiment2}
\small{
Conditional linear regression with $d = 1$ in the $y \times z$ plane: for all data, $y \in [-\pi,\pi]$, $z = \sin(y) + noise$,  $dim(\vec{x})=3$. $x_1, x_2, x_3$ are defined as ``$y \leq -\pi/2, 0 , \pi/2$'' respectively. The algorithm gives three candidates: the red one is the region $y < 0$; the green one $y \leq 0$; the blue one $ -\pi/2 \leq y \leq \pi/2 $.
}\raggedright
\end{figure}

\subsection{Example 3: Scale Up}
In this synthetic data experiment, we set  $dim(\vec{x}) = 7, dim(\vec{y}) = 10$, and generated $N = 100000$ points in total with $\mu = 0.5$. We generate $\Wstar$ uniformly from $[-10,10]$ and randomly choose a $4$-term $2$-DNF to use as the condition describing the good data. For the bad data, $\YI \in [-1,1], \ZI \in [-10,10]$. For good data, $\YI \in [-1,1]$ and $\ZI = \langle \YI, \Wstar \rangle + noise$ with variance $100$.
We set $S = 0.1, \gamma = 0.1$, $r = 100$ and $r_{final} = 1$. In several trials, each time our algorithm finds several pairs of regression parameters and DNFs, with one of them containing our planted DNF. We note that there are also other pairs in the output with even lower error, meaning the algorithm also finds other subsets on which the error is smaller than for our planted solution. Note that since the previous algorithms~\citep{juba17,hjlw19} scale exponentially with $d$, such an instance would be infeasible to solve using those algorithms.

\subsection{Example 4: Real World Data}
We test our algorithm on two of the larger benchmark data sets from the LIBSVM repository \cite{libsvm}, Space and Cpusmall, used previously by \cite{hjlw19}, and compared the loss achieved for several target fractions. These standard data sets contains only  real-valued attributes. Following Hainline et al's strategy, we generate Boolean attributes using indicators for membership in the empirical $50\%$-quantile of each real attributes. In other words, for each $i$th point, its $j$th Boolean attribute is defined as
\[x^{(i)}_j = 
\begin{cases}
1 \quad if \ y^{(i)}_j \geq median(y_j)\\
0 \quad if \ y^{(i)}_j < median(y_j)
\end{cases}
\]
We randomly selected $1/3$ of the data as the training data and the other $2/3$ as the testing data. Similar to \cite{hjlw19}, we use the algorithm to find a (list of) DNF on the training data, compute the linear regression rule on the corresponding subset on the training data, and then test the loss of the regression rule on the corresponding subset of the testing data. Partially to cope with instability we observed in the SDP solver, we decreased the number of padded decompositions (in Algorithm~\ref{algo4}) from $112\log(\frac{l(l+1)}{\delta})$ to $12$,
and instead repeated the algorithm 50 times to produce a list of candidates.
We compute the linear fit and training loss for each candidate DNF, and then use the DNF with lowest training loss as our final output to test on the testing data. We repeated the experiments using different values of $\mu$ ($0.2, 0.4, 0.6, 0.8, 1.0$) and set the parameters $S = \mu \cdot 10$, $\gamma = 0.1$, and $r_{final} = S \cdot dim(y)$, where $dim(y)$ is the dimension of the real features. 

\begin{figure}[h]
\includegraphics[width=8cm]{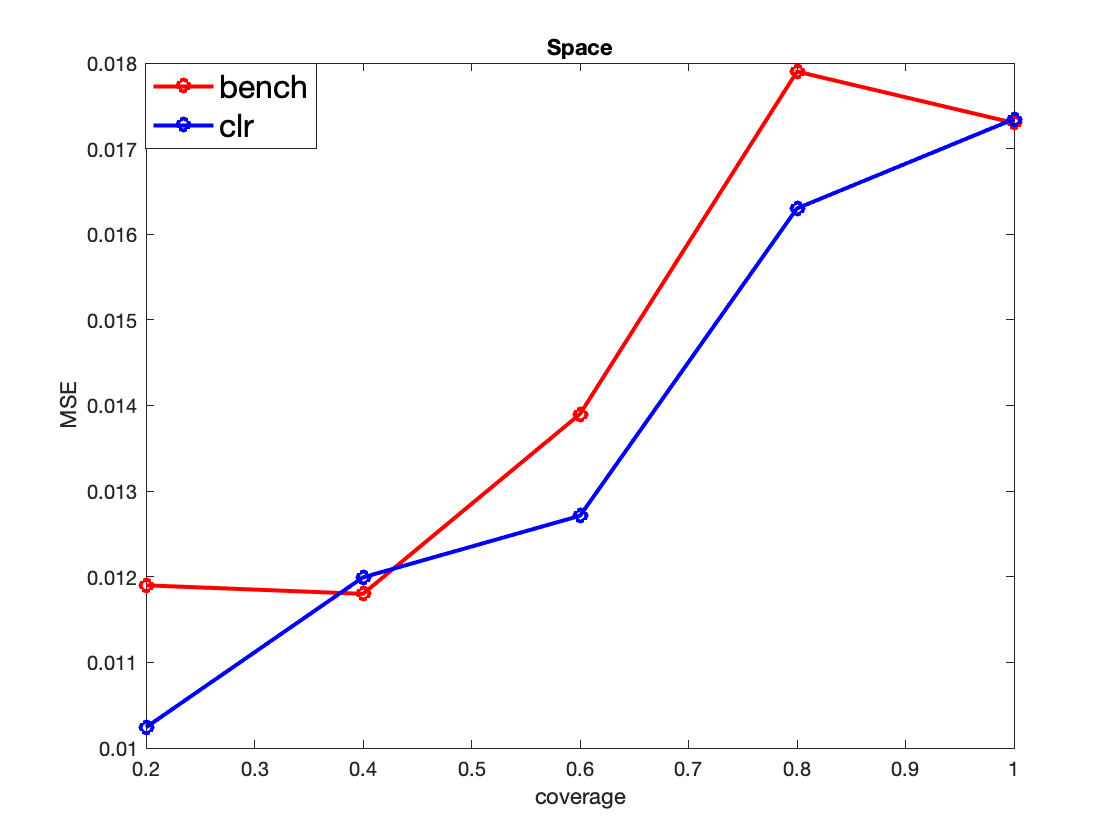}
\includegraphics[width=8cm]{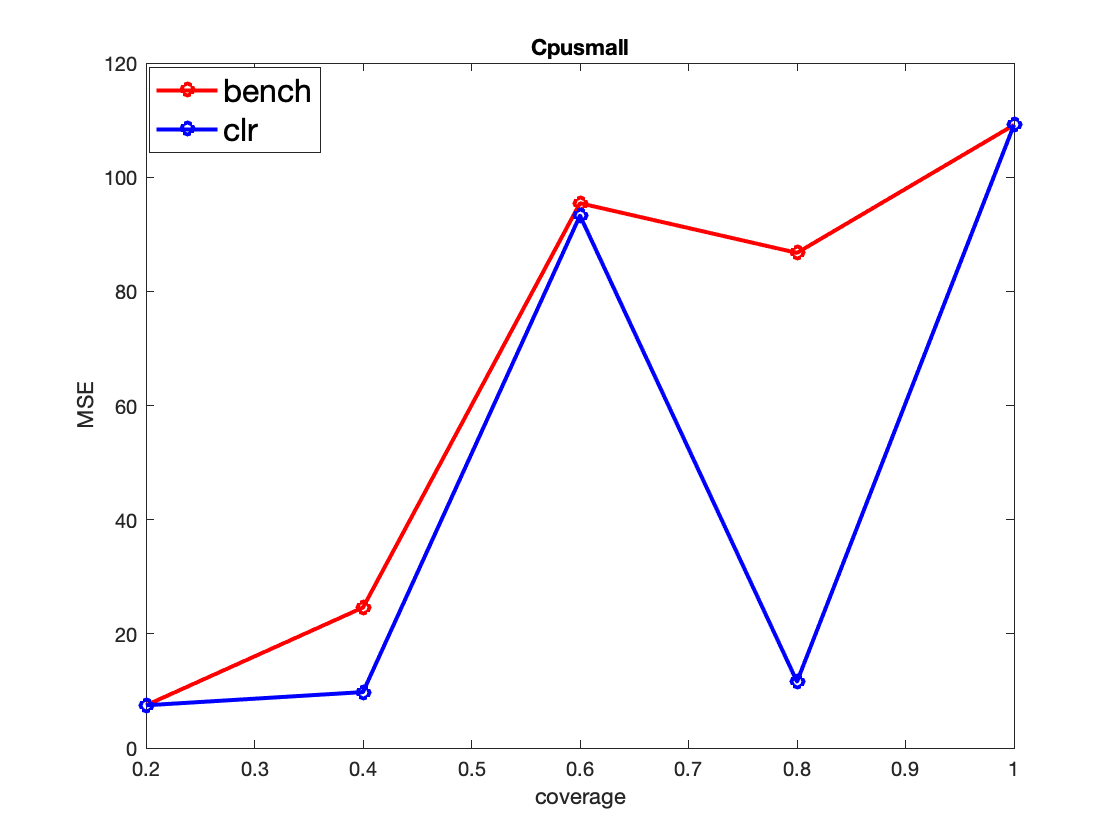}
\centering
\caption{Linear regression on LIBSVM datasets}\label{experiment}
{\bf Left: Space data}:$N = 1025, dim(x) = dim(y) = 6$. 
{\bf Right: Cpusmall data}:  $N = 2703, dim(x) = dim(y) = 12$.
\end{figure}
As shown in Figure \ref{experiment}, our performance is better than benchmark algorithm for most of the cases, and at least comparable in all cases. 
For $\mu = 1$, conditional linear regression is the same as a simple linear regression, so unsurprisingly, our algorithm and Hainline et al's algorithm have the same result.
We note that our running time is much better than Hainline et al's algorithm, even on these small data sets. Their algorithm required a few days of cloud computing time, while our algorithm only required a few hours.

\subsection{Discussion of the Experiments}
These experiments illustrate the variety of tasks our algorithm can solve. They also demonstrate that our algorithm can work in practice: The algorithm is feasible to run, in spite of the fact that it must solve numerous SDP optimization problems. It does find the desired, planted answers on the synthetic data where this is known, and it obtains comparable or superior error to the baseline algorithm of \cite{hjlw19}. We obtain a computational advantage over \cite{lud} because we group the data points into terms and precompute the loss functions for these terms. So, instead of solving a huge SDP problem with variables for each example as in Charikar et al's approach, our SDPs are potentially of reasonable size and can run in a fraction of a second. The radius $r$ of the candidate parameters shrinks exponentially over the outer loop of the list regression algorithm, so even if we set $r_{final}$ to be $0.1$ or so, it terminates within $10$ iterations. These improvements yield an algorithm that is feasible to run in practice, so long as the number of regression features and candidate terms is only moderately large.

One difficulty with the use of our algorithm is that we need to specify the setting of several parameters. One has to guess an estimate of $\mu$ and $S_0$ to run the algorithm, so in practice one may need to try a sequence of candidate settings for these parameters.
The smaller $S$ and $r_{final}$ are set, the more accurate an estimate the algorithm can provide, at the cost of more computation time.
In practice, we cannot set $r_{final}$ to be too small, as otherwise the algorithm tends to think there does not exist a candidate with size $\mu$. We believe that this may be a consequence of noise in the data.

It is interesting that, on the Cpusmall dataset, the loss obtained at $\mu = 0.8$ is much smaller than the loss we obtained at $\mu = 0.6$. Usually we would expect the loss will decrease with $\mu$: since our problem statement allows us to output a solution on which the condition comprises 80\% of the data when we are only seeking 60\%, it is strictly easier to fit a smaller subset. We note that this difference is not caused by the double-counting of points, since the double-counted loss of $\mu = 0.8$ is still smaller than the loss of $\mu = 0.6$. Thus, it seems that the problem is that the algorithm does not obtain such accurate estimates of the regression parameters on $\mu=0.6$. We note that the value of $\mu$ enters Algorithm~\ref{algo1} in several places, and the guarantees we can provide on its output feature a $1/\sqrt{\mu}$ factor blow-up in the error; notice, our Theorem \ref{2.2} requires a stronger condition for smaller $\mu$). Indeed, in the soft regression/outlier detection procedure, we might intuitively expect that as the proportion of ``signal'' decreases, we might get less accurate estimates. Thus, in practice, we would advise running the algorithm several times using $\mu=
(1-\Delta), (1-\Delta)^2,\ldots,
(1-\Delta)^i$ for some number of iterations $i$ ($\Delta\in (0,1)$)  down to a desired minimum, and taking the condition/parameters with the smallest loss.

One may worry that for small $\mu$, there is a chance of overfitting.
However, this is not a problem as long as we have a reasonably large data set and enough runs. We observe that for $\mu = 0.2$, $10$ out of the $50$ runs do have much larger testing error than training error, which implies overfitting. But when we pick the candidate of the minimum training error, overfitting doesn't impact our output. To prevent overfitting, we suggest to use larger data sets and run more trials so that even small subsets have lower variance.

We also caution that our algorithm requires enough data to achieve stability. We also examined some of the smaller LIBSVM data sets (bodyfat and Boston housing) that have only a few hundred examples. For these small data sets, we found that the SDP solver could not solve instances in which the radius of the parameter space decreased below 1. Thus, the final iteration of parameters we obtained provided poor estimates that were not competitive with the algorithm of \cite{hjlw19}. In conclusion, our algorithm is better in terms of both running time and accuracy than the previous algorithm for the larger benchmark data sets ($N\geq 1000$). Since it collapses the data points into loss matrices for terms, and the corresponding pre- and post-processing can be done in roughly linear time, it can scale up to very large data sets. But, it is not appropriate to use for very small data sets.

\section{Directions for future work}
Our work has introduced the first algorithm for conditional linear regression that simultaneously obtains conditions of near-optimal probability, obtains an $\Oo(\Mgood \log \mu N)$-approximation to the optimal loss, and does not feature a running time that depends exponentially on the number of factors used in the linear predictor. Theoretically, the one ``additional cost'' imposed by our approach was a requirement that ${\Sep}_0$, the spectral difference between the covariances of the individual terms and the overall desired subset of the data we are interested in, is sufficiently small compared to the convexity of our loss functions. Although we noted that we can always impose sufficient convexity by adding a regularization term, in some cases this may yield insufficiently accurate estimates of the regression parameters. Thus, the main theoretical question raised here is whether or not this requirement can be relaxed or removed entirely. We stress that the previous approaches did not require such an assumption, which suggests that it may be possible to remove it.

Another family of questions concerns the amount of blow-up we incur over the loss. While our bound seems quite good when the number of terms in the desired DNF $\Mgood$ is small, one can always ask whether or not it is optimal. We do not have techniques for answering this question at the moment. At the same time, we recall that \citet{zmj17} obtain a better approximation factor for the condition search (learning abduction) problem for large formulas: by more carefully controlling the amount of double-counting, Zhang et al.~obtain a better blow-up when $\Mgood >> n^{k/2}$ for $k$-DNFs on $n$ Boolean attributes. Thus, another natural question is if we can likewise quadratically reduce the blow-up of the loss, as achieved by Zhang et al.~for the condition search  problem.

One can also seek to further improve the efficiency of our algorithm in several respects. For one, we have made no attempt to optimize the amount of data required for our guarantees. (Or, for that matter, required by the algorithm.) For two, we believe that there is substantial scope for improving the running time of our algorithm. \citet{dks18} recently made several improvements to the algorithm of \citet{lud}, replacing the semidefinite programs with eigenvector computations and using a simpler clustering strategy than padded decompositions. It may be possible to analogously improve the running time of our algorithm.

Finally, we note that \citet{lud} had intended their algorithm for use in solving a variety of tasks, and thus gave a generic analysis in terms of abstract loss functions. This helped enable our work, although our results depend on key features of the squared-error loss for regression---in particular, we needed that we could show that our loss decreased linearly with the radius. Then, since we could give a bound on the radius that scaled linearly with the previous iteration's radius, we could iteratively improve our estimates by simply running another iteration of the algorithm. There are conceptually similar tasks, such as conditional {\em linear classification}, where we seek to make Boolean (as opposed to real-valued) predictions, that we might hope to be able to solve using an algorithm of this form, but our approach would require that linear scaling condition holds for the classification loss function, which is not clear. It would be interesting to see whether and how our approach could be adapted to solve such problems.

\section*{Acknowledgements}
We thank Jacob Steinhardt for sharing his code and many helpful discussions. We also thank Ilias Diakonikolas for a helpful discussion. We finally thank our anonymous reviewers for their comments on related works.

\bibliographystyle{icml2019}
\bibliography{robust}

\end{document}